\documentclass{article}

\PassOptionsToPackage{numbers,compress}{natbib}

 \usepackage[preprint]{neurips_2026}

\usepackage[utf8]{inputenc} 
\usepackage[T1]{fontenc}    
\usepackage{hyperref}       
\usepackage{url}            
\usepackage{booktabs}       
\usepackage{multirow}       
\usepackage{amsmath}        
\usepackage{amsthm}         
\usepackage{amssymb}        
\usepackage{amsfonts}       
\usepackage{nicefrac}       
\usepackage{microtype}      
\usepackage{xcolor}         
\usepackage{graphicx}      
\usepackage{caption}       

\newtheorem{proposition}{Proposition}

\newtheorem{hypothesis}{Hypothesis}
\title{Unlocking Complex Visual Generation via Closed-Loop Verified Reasoning}

\author{%
  Hanbo Cheng$^{1}$ \And
  Limin Lin$^{2}$ \And
  Ruo Zhang$^{2}$ \And
  Yicheng Pan$^{1}$ \And
  Jun Du$^{1}$ \AND
  \normalfont
  $^{1}$University of Science and Technology of China (USTC) \\
  $^{2}$Independent Researcher \\
  \small project page: \url{https://hanbo-cheng.github.io/CLVR_Proj/}
}

\begin{document}

\maketitle

\vspace{-0.25em}
\begin{center}
  \noindent
  \begin{minipage}[c]{0.4995\linewidth}
    \includegraphics[width=\linewidth]{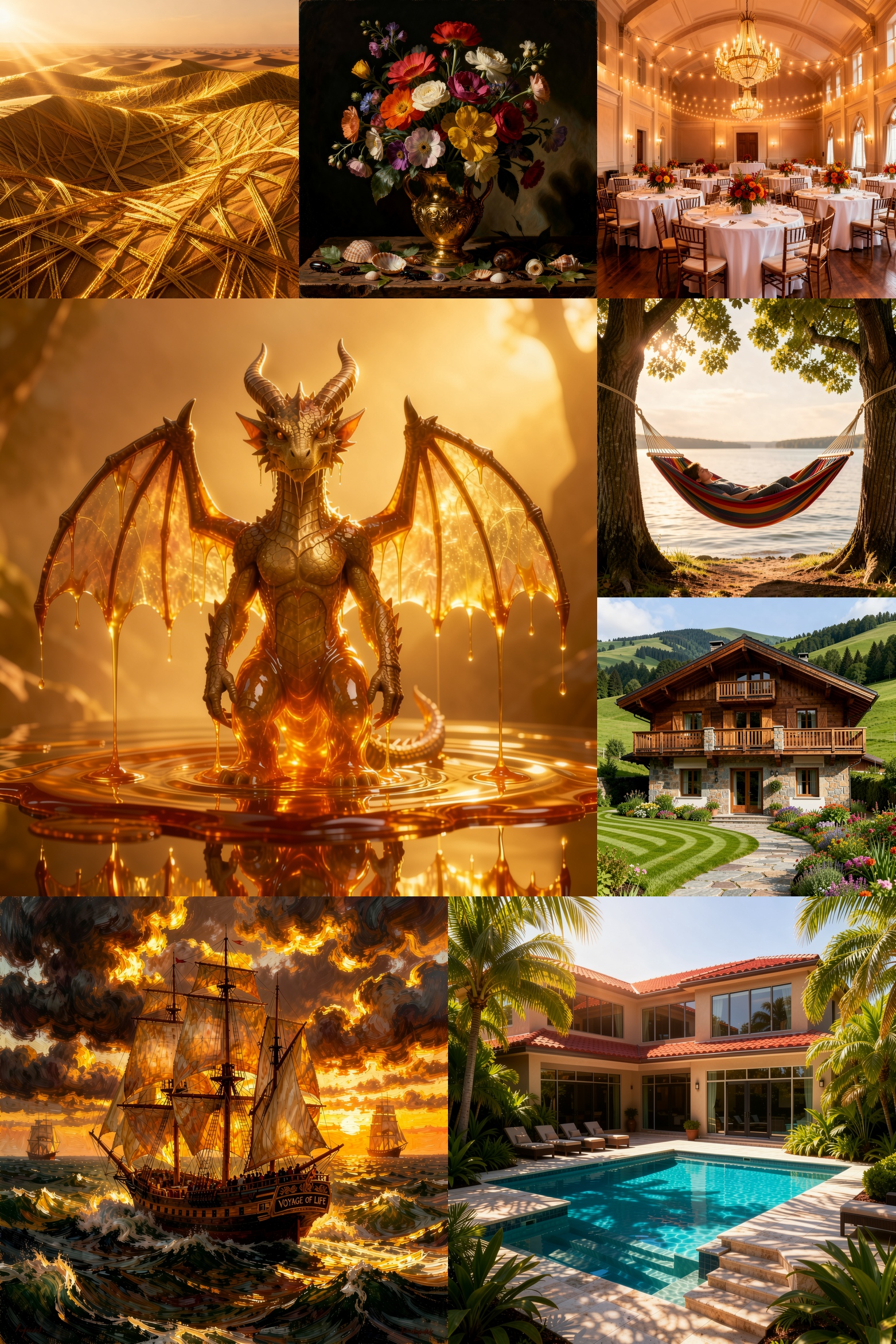}%
  \end{minipage}%
  \begin{minipage}[c]{0.4995\linewidth}
    \includegraphics[width=\linewidth]{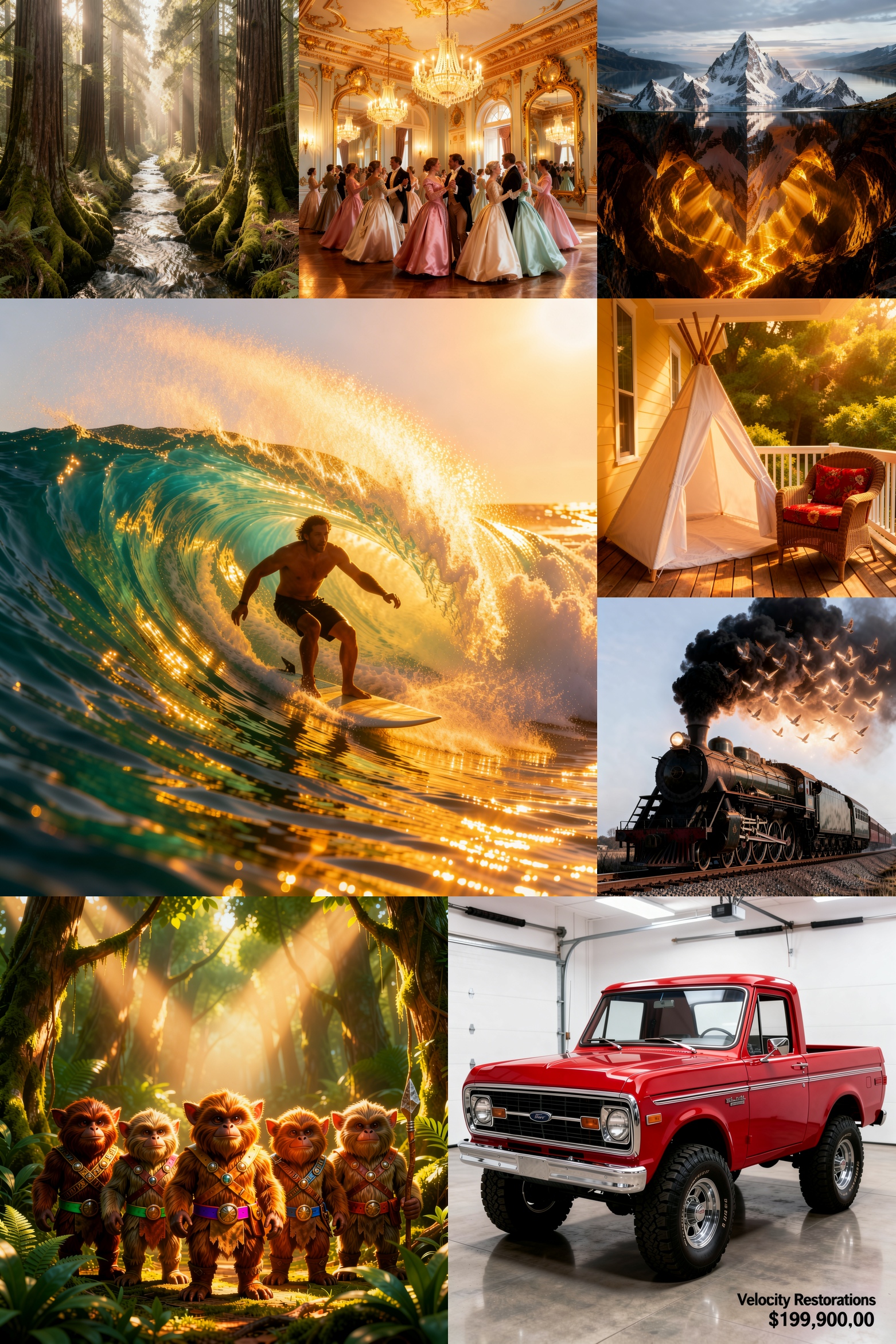}%
  \end{minipage}\\[0.45em]
  \begingroup
  \captionsetup{hypcap=false}%
  \captionof{figure}{Qualitative results of CLVR. The prompts are from the PRISM benchmark \cite{PRISM-Bench}.}%
  \label{fig:gallery_combined}%
  \endgroup
\end{center}
\vspace{0.35em}

\begin{abstract}
  Despite rapid advancements, current text-to-image (T2I) models predominantly rely on a single-step generation paradigm, which struggles with complex semantics and faces diminishing returns from parameter scaling. 
  While recent multi-step reasoning approaches show promise, they are hindered by ungrounded planning hallucinations lacking verification, monolithic post-hoc reflection, long-context optimization instabilities, and prohibitive inference latency. To overcome these bottlenecks, we propose the Closed-Loop Visual Reasoning (CLVR) framework, a comprehensive system that deeply couples visual-language logical planning with pixel-level diffusion generation. CLVR introduces an automated data engine with step-level visual verification to synthesize reliable reasoning trajectories, and proposes Proxy Prompt Reinforcement Learning (PPRL) to resolve long-context optimization instabilities by distilling interleaved multimodal histories into explicit reward signals for accurate causal attribution. Furthermore, to mitigate the severe latency bottleneck caused by iterative denoising, we propose $\Delta$-Space Weight Merge (DSWM), a theoretically grounded method that fuses alignment weights with off-the-shelf distillation priors, reducing the per-step inference cost to just 4 NFEs without requiring expensive re-distillation. Extensive experiments demonstrate that CLVR outperforms existing open-source baselines across multiple benchmarks and approaches the performance of proprietary commercial models, unlocking general test-time scaling capabilities for complex visual generation.
\end{abstract}



\section{Introduction}

In recent years, text-to-image (T2I) generation models have made remarkable progress in visual quality and realism \cite{QwenImage, Seedream4, Z-Image, longcat}. However, current T2I systems predominantly follow a "single-step generation" paradigm, attempting to map all textual instructions to pixels in a single forward pass. While effective for simple prompts, this approach often struggles with complex inputs---leading to attribute confusion, missing entities, or misaligned spatial relations \cite{geneval,T2I-CompBench, Echo-4o, niu2025wise}. This indicates that the single-step generation paradigm faces an empirical capacity ceiling when handling complex semantics.

Through a controlled complexity-stratified probing study, we observed that as semantic complexity increases, advanced single-step models inevitably suffer from structural degradation (see Section \ref{sec:probe}). While increasing model capacity offers some relief, it yields diminishing marginal returns: achieving linear capability gains typically demands exponential increases in parameters and compute \cite{Scaling_law}. Such disproportionate costs imply that scaling alone may not be the most efficient or sustainable route to achieving precise semantic alignment.
Recently, the integration of Chain-of-Thought (CoT) reasoning has led to substantial improvements in the performance of Large Language Models and Vision-Language Models (LLM/VLM) on complex logic and planning tasks \cite{o1, DeepSeek-R1}. Inspired by this paradigm shift, a natural question arises: can a similar CoT approach be extended to image generation? This has motivated the transition from traditional one-step generation toward a reasoning-based generation paradigm, where complex visual objectives are achieved through a sequential, CoT-style generative process.

However, transitioning such closed-loop visual reasoning from a conceptual framework to practical systems still faces four major technical challenges. First, \textbf{a lack of high-quality verified data}: existing synthesis methods for visual Chain-of-Thought (CoT) trajectories often lack rigorous verification. Consequently, while introducing a thinking process improves the final output, the intermediate reasoning steps are typically ungrounded and error-prone, which severely limits the overall effectiveness of CoT \cite{UniCoT,IRG}. Second, \textbf{inadequate task decomposition}: current text-to-image CoT paradigms predominantly rely on post-hoc reflection rather than breaking down complex prompts into simpler, manageable sub-tasks. As a result, the final generation quality remains largely predetermined by the initial generation step \cite{UniCoT,Process_Driven}. Third, \textbf{multimodal long-context optimization}: visual CoT inherently introduces long, interleaved image-text contexts. Models easily become confused by such extended inputs, fundamentally reflecting a lack of multimodal understanding capability under existing training paradigms. Finally, \textbf{architectural coupling and inefficiency}: many approaches \cite{TWiG, Vinci, Loom} rely on Unified Multimodal Models (UMMs) \cite{Janus-Pro,Bagel} to process multimodal outputs simultaneously, leading to slow inference speeds. Furthermore, this reliance on UMMs prevents these methods from seamlessly leveraging the rapid, independent advancements of standalone Vision-Language Models (VLMs) and Diffusion base models \cite{DraCo, MURE}.

To address these challenges, we propose the Closed-Loop Visual Reasoning (CLVR) framework that fully connects data synthesis, model alignment, inference mechanisms, and deployment acceleration. The main contributions of this paper are as follows:

\begin{enumerate}
  \item \textbf{CLVR Paradigm for General Test-Time Scaling}: To tackle the inadequate task decomposition and multimodal long-context optimization instabilities, we propose the Closed-Loop Visual Reasoning (CLVR) framework for text-to-image generation. Specifically, by introducing Proxy Prompt Reinforcement Learning (PPRL) to achieve stable optimization over extended multimodal contexts, our method successfully unlocks more general test-time scaling capabilities in visual generation tasks.

  \item \textbf{Automated Data Engine for Verified Trajectories}: To address the lack of high-quality verified data for visual CoT, we propose a fully automated data production framework capable of generating verified, high-quality CLVR trajectories. This establishes a solid data foundation for test-time scaling in visual generation.

  \item \textbf{$\Delta$-Space Weight Merge (DSWM) for Fast Inference}: To overcome the architectural inefficiency and severe latency bottlenecks of iterative reasoning, we introduce DSWM, a method that leverages distillation priors to accelerate CLVR inference. Supported by theoretical analysis and ablation results, DSWM achieves promising speedups, transforming multi-step visual reasoning from a theoretical framework into a practically deployable solution.

  \item \textbf{System-Level Cross-Benchmark Improvements}: Across multiple evaluated benchmarks, CLVR outperforms most open-source baselines included in our comparison and narrows the gap to proprietary models.
\end{enumerate}

\section{Related work}
\paragraph{Reasoning-enhanced Text-to-Image Generation}
Existing approaches attempt to improve complex semantic alignment through pre-planning \cite{ImageGen_CoT, Coco} or interleaved reasoning and reflection \cite{IRG, ReflectionFlow}. However, these methods suffer from two primary technical limitations. 
First, verification in existing trajectory-construction pipelines is often insufficient: many training examples still contain diffusion-side execution failures, so supervision implicitly mixes reliable steps with erroneous rollouts and biases learning toward post-hoc error correction rather than planning within verifiably executable bounds \cite{UniCoT}. 
Second, information decay in extended histories causes the model to lose track of global constraints, leading to inconsistent outputs over multiple iterations \cite{Loom}.

\paragraph{Unified Multimodal Generation Models}
Unified Multimodal Models (UMMs) integrate understanding and generation within a single architecture \cite{Show-o, Janus-Pro, Bagel}. While UMMs offer native multimodal processing for CoT reasoning \cite{TWiG, DraCo, IRG}, their tightly coupled parameters result in substantial joint training costs. More importantly, this monolithic design prevents the system from leveraging the rapid iterative advancements of independent VLM and diffusion base models, causing the overall capability growth to lag behind specialized state-of-the-art foundations.

\paragraph{Diffusion Alignment and Distillation}
Current preference alignment \cite{DiffusionNFT, Flow-GRPO} and distillation techniques \cite{CM, LCM, DMD} are primarily optimized for direct text-to-image generation. In multi-step reasoning contexts, existing RL-based alignment struggles because traditional reward models lack the capacity to interpret and evaluate the interleaved logic within complex multimodal histories, leading to reward collapse. Furthermore, the scarcity of specialized trajectory data makes re-distilling these closed-loop systems impractical \cite{PCM}.

\section{Method}
\label{sec:method}

In this section, we present the Closed-Loop Visual Reasoning (CLVR) framework (Figure~\ref{fig:clvr_overview}). Our framework comprises three core components: (1) Trajectory Synthesis: We employ a state-constrained controller with step-level validation to generate reliable, interleaved CoT trajectories. (2) Diffusion Alignment: We introduce Proxy Prompt Reinforcement Learning (PPRL) to achieve stable optimization over extended multimodal contexts. (3) Efficient Deployment: During inference, we utilize trajectory-accumulative conditioning for historical consistency and propose $\Delta$-Space Weight Merge (DSWM), a theoretically grounded method that fuses alignment weights with distillation priors to achieve substantial acceleration without re-distillation.

\begin{figure}[t]
  \centering
  \includegraphics[width=\linewidth]{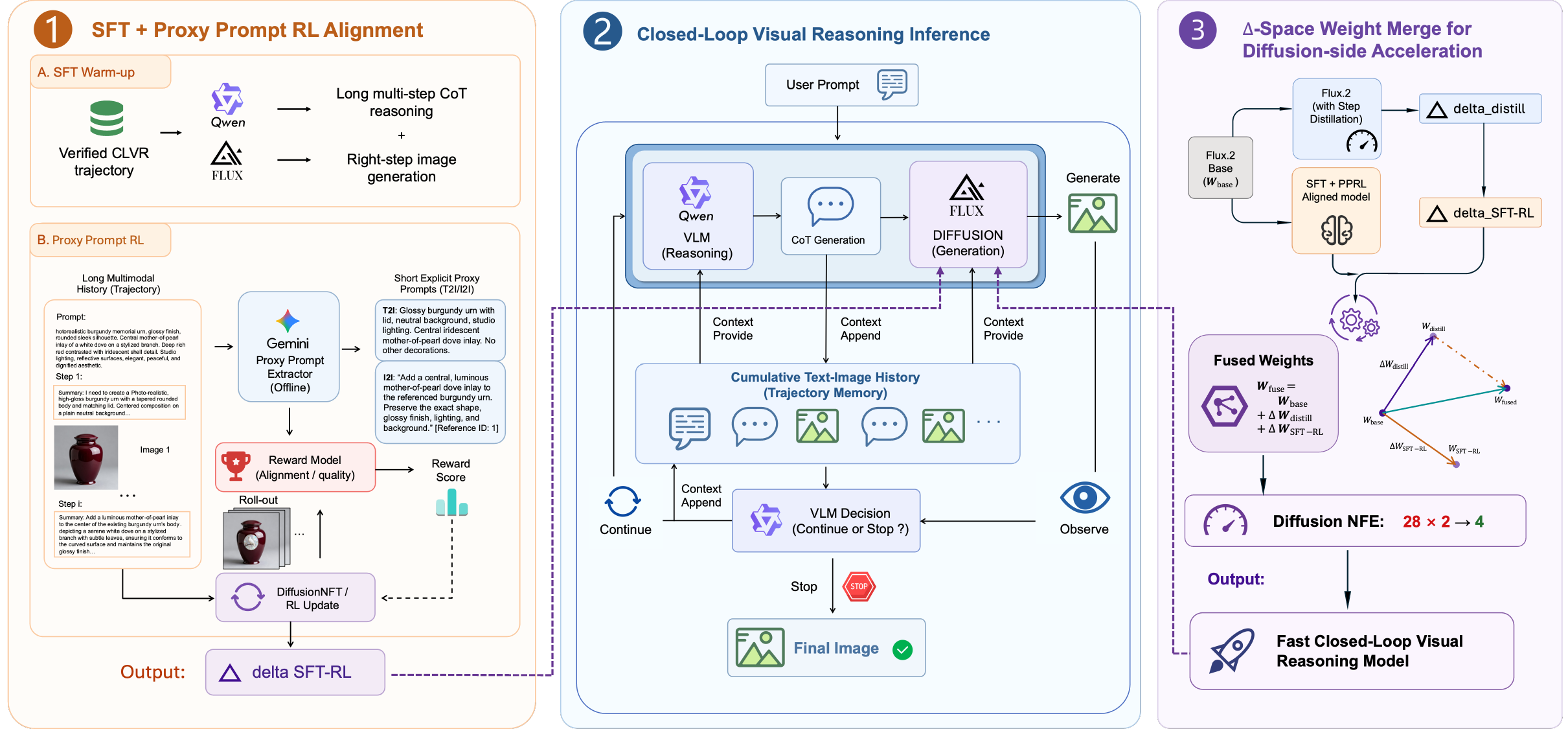}
  \caption{Overview of the CLVR framework. The pipeline consists of three main components: (1) Flowchart of SFT and Proxy Prompt Reinforcement Learning; (2) Pipeline of the CLVR inference framework; (3) Schematic of the $\Delta$-Space weight merge algorithm. Finally, the merged model from (3) is integrated into the inference framework to achieve efficient CLVR inference.}
  \label{fig:clvr_overview}
\end{figure}

\subsection{Closed-Loop Visual Reasoning Data Synthesis}
\label{sec:datapipeline}

To prevent the cascading failures commonly observed in multi-step generative processes, we design a verification-centric data engine for CLVR. As illustrated in Figure \ref{fig:datapipeline}, this pipeline transforms the framework into a practical, scalable system by providing high-fidelity, verified reasoning trajectories. We conceptualize the VLM as a closed-loop controller following a Reason-to-Act paradigm \cite{ReAct}. At each step, it assesses the canvas, reasons about semantic gaps, and enacts decisions by invoking discrete tools (e.g., Initial Generation, Image Editing, Result Validation, or Trajectory Termination).

Crucially, to ensure robustness without compromising model capacity, our data engine features a dual-track verification mechanism:
\begin{itemize}
    \item \textit{Passive verification} acts as a step-level gatekeeper. After every generative tool call, a sub-agent confirms whether the diffusion model successfully executed the given instruction via a dynamically generated checklist. If a step fails, we interpret it as exceeding the diffusion model's inherent capacity and immediately discard the entire trajectory context to restart from scratch. This strict filtering ensures that no generative errors contaminate the final dataset.
    \item \textit{Active verification} serves as the global error-correction hub. It is explicitly invoked by the controller to validate whether the current canvas aligns with the user prompt. If semantic gaps are detected, it provides actionable feedback, allowing the controller to dynamically adjust its plan and re-execute prior steps, thereby closing the reasoning loop.
\end{itemize}

Beyond step-level and interactive validation, candidate trajectories undergo consensus-based global filtering. We generate a single-step baseline and conduct a blind A/B comparison evaluated by two independent judge VLMs (Gemini 2.5 Pro \cite{Gemini25}, Seed 1.8 \cite{seed2026seed18modelcardgeneralized}). A trajectory is retained only if both judges agree that the multi-step CoT result achieves superior instruction following and visual quality.

Finally, during the execution-to-reasoning translation phase, we convert the discrete execution logs into coherent natural language CoT narratives. This preserves temporal consistency, critical observations, and feedback-driven corrections, making the raw tool sequences directly suitable for model alignment (shown in Appendix, Figure \ref{fig:cot_case}). See Appendix \ref{sec:appendix_data_pipeline} for a detailed description of the CLVR data pipeline.

\begin{figure}[t]
  \centering
  \includegraphics[width=\linewidth]{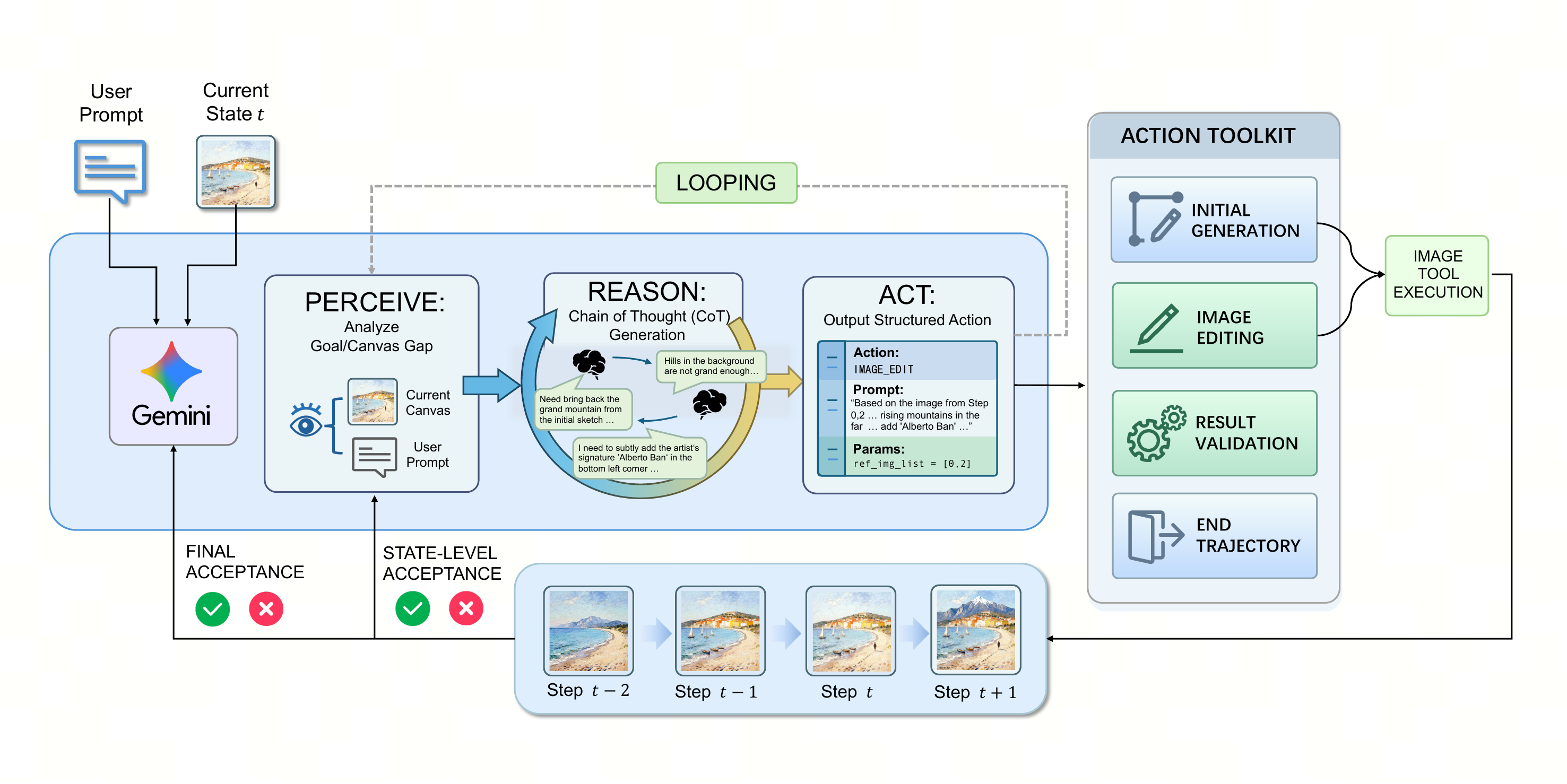}
  \caption{Architecture of the CLVR data synthesis pipeline, featuring a Perceive-Reason-Act workflow. The system is controlled by Gemini 2.5 Pro \cite{Gemini25}.}
  \label{fig:datapipeline}
\end{figure}

\subsection{Proxy Prompt Reinforcement Learning}

Building on the verified trajectories from Section~\ref{sec:datapipeline}, we propose a two-stage alignment pipeline: Supervised Fine-Tuning (SFT) followed by Proxy Prompt Reinforcement Learning (PPRL). We first employ standard SFT as a warm-up training to adapt both the VLM and diffusion model to multi-step planning. This transitions the models from short-prompt priors to interleaved reasoning trajectories, establishing a robust policy initialization for subsequent RL.

\paragraph{Training Data Truncation}

To construct the training objective for closed-loop visual reasoning, we utilize offline ground-truth reasoning trajectories. Given a complete trajectory $\mathcal{T} = \{(r_0, \mathbf{x}_0), \dots, (r_T, \mathbf{x}_T)\}$, we truncate it at an arbitrary step $t \in [0, T]$. This truncation yields the local multimodal context $\mathbf{c}_t$, which represents the history prior to the current generation:
\begin{equation}
    \mathbf{c}_t = \{x_{\text{prompt}}, (r_0, \mathbf{x}_0), \dots, (r_{t-1}, \mathbf{x}_{t-1}), r_t\},
\end{equation}
where $x_{\text{prompt}}$ is the initial user goal, $r_i$ denotes the textual reasoning, and $\mathbf{x}_i$ is the generated image at step $i$. By treating $\mathbf{c}_t$ as the conditional input and $\mathbf{x}_t$ as the optimization target, we can explicitly train the diffusion policy to generate accurate images conditioned on lengthy, incremental visual states.

\paragraph{Proxy Prompt Mechanism}

To stabilize alignment over extended multimodal contexts, we introduce the \textbf{Proxy Prompt} mechanism, as shown in Figure \ref{fig:clvr_overview} (1). In multi-step visual reasoning, directly evaluating generated images against long-range interleaved histories often introduces significant reward noise, as standard reward models are typically optimized for short, explicit instructions rather than verbose Chain-of-Thought trajectories.

To bridge this gap, we employ a powerful foundation VLM (denoted as $f_{\text{VLM}}$) as an offline teacher to distill the complex history $\mathbf{c}_t$ into explicit, evaluable instructions, the proxy prompts. For both initial generation ($t=0$) and subsequent image editing ($t>0$), the extraction process is formalized as:
\begin{equation}
\begin{cases}
p_{\text{T2I}} = f_{\text{VLM}}(\mathbf{c}_t), & \text{if } t = 0 \text{ (Initial Generation)} \\
(p_{\text{T2I}}, p_{\text{I2I}}, \mathbf{I}_{\text{ref}}) = f_{\text{VLM}}(\mathbf{c}_t), & \text{if } t > 0 \text{ (Image Editing)}
\end{cases}
\end{equation}
where $p_{\text{T2I}}$ denotes the comprehensive scene description, $p_{\text{I2I}}$ represents the specific editing instruction, and $\mathbf{I}_{\text{ref}}$ is a list of indices for reference images selected by the VLM from the historical image set $\mathbf{C}_{\text{img}} \in \mathbf{c}_t$.

The final proxy reward $R_{\text{proxy}}$ combines a global quality reward model ($R_{\text{T2I}}$) and an editing reward model ($R_{\text{I2I}}$), calculated as follows:
\begin{equation}
R_{\text{proxy}}(\mathbf{c}_t, \mathbf{a}) =
\begin{cases} 
R_{\text{T2I}}(\mathbf{a}, p_{\text{T2I}}), & \text{if } t = 0 \\
0.5 \cdot R_{\text{T2I}}(\mathbf{a}, p_{\text{T2I}}) + 0.5 \cdot R_{\text{I2I}}(\mathbf{a}, \mathbf{C}_{\text{img}}[\mathbf{I}_{\text{ref}}], p_{\text{I2I}}), & \text{if } t > 0
\end{cases}
\end{equation}

By utilizing proxy prompts, we essentially distill the long-context understanding capabilities of the foundation VLM into the RL reward signal via natural language and reference image indices. Upon obtaining $R_{\text{proxy}}$, we employ the DiffusionNFT algorithm \cite{DiffusionNFT} for step-wise policy optimization. Specifically, we use $R_{\text{proxy}}$ as the reward feedback to guide the diffusion model toward the high-quality generation distribution defined by the proxy prompts, while maintaining the SFT prior knowledge through KL constraints.

\subsection{Closed-Loop Visual Reasoning Inference}

To maintain global consistency and retain critical constraints across multiple reasoning steps, we formulate the inference pipeline as an interactive agentic workflow. 
This framework deploys the Visual Language Model (VLM) as an autonomous router policy $\pi_{\text{VLM}}$ and the diffusion model as a context-aware generator $\mathcal{D}_{\text{gen}}$, establishing a multi-turn, self-feedback execution loop. The core mechanism involves trajectory-accumulative conditioning, where the context fed to the diffusion model dynamically maintains the full reasoning trace rather than just the initial prompt. This empowers the diffusion model to deeply comprehend long-horizon dependencies and complex instructions, a capability explicitly enhanced through our PPRL optimization.

The CLVR workflow is depicted in Figure~\ref{fig:clvr_overview} (2). Specifically, at each iteration $t$, the VLM evaluates the current canvas state $\mathbf{x}_{t-1}$ alongside the accumulated multimodal history $\mathbf{c}_{t-1} = \{x_{\text{prompt}}, (r_0, \mathbf{x}_0), \dots, (r_{t-1}, \mathbf{x}_{t-1})\}$. It then formulates an action plan by sampling a reasoning narrative $r_t$ and a discrete action signal $a_t$ according to $(r_t, a_t) \sim \pi_{\text{VLM}}(\cdot | \mathbf{c}_{t-1})$. If the VLM determines that the canvas requires further modification (i.e., $a_t = \texttt{<|image\_gen|>}$), it dispatches the generation task to the diffusion model. The condition state is updated to $\mathbf{c}_t = \mathbf{c}_{t-1} \cup \{r_t\}$, and the diffusion model, leveraging its enhanced long-context understanding, synthesizes a new refined image $\mathbf{x}_t = \mathcal{D}_{\text{gen}}(\mathbf{c}_t, \mathbf{x}_{t-1})$. This new image is then appended to the history, and the loop advances to the next round of inspection. Conversely, if the VLM judges that the current image sufficiently fulfills the user goal (i.e., $a_t = \texttt{<|terminate|>}$), it triggers a termination signal and outputs the current canvas $\mathbf{x}_{t-1}$ as the final result. 

\subsection{\texorpdfstring{$\Delta$-Space}{Delta-Space} Weight Merge for Deployable Reasoning}
\label{sec:DSWM}
To achieve deployable inference speeds, diffusion models typically rely on step distillation. However, applying standard re-distillation to reasoning-specialized models is impractical due to the prohibitive cost of constructing large-scale, high-quality Chain-of-Thought (CoT) trajectory data. To bypass this data bottleneck, we propose directly reusing off-the-shelf T2I/I2I distillation priors via parameter merging, based on a geometric decoupling analysis.

\paragraph{Theoretical Analysis: Normal-Tangent Approximate Decoupling}
We explore the mathematical feasibility of linearly fusing existing distilled weights ($\Delta \mathbf{W}_{\text{distill}}$) with newly learned closed-loop alignment weights ($\Delta \mathbf{W}_{\text{Align}} = \Delta \mathbf{W}_{\text{SFT}} + \Delta \mathbf{W}_{\text{RL}}$). Let the base diffusion model be $f(\mathbf{W}_{\text{base}})$. 

\begin{proposition}[Linear Superposition of First-Order Perturbations] \label{prop:linear_superposition}
Assuming the parameter variations introduced by fine-tuning reside within the local linear perturbation region, the output increment of the fused model can be approximately decomposed as the linear superposition of independent task increments:
\begin{equation}
f(\mathbf{W}_{\text{base}} + \Delta \mathbf{W}_{\text{distill}} + \Delta \mathbf{W}_{\text{Align}}) \approx f(\mathbf{W}_{\text{base}}) + \Delta \mathbf{f}_{\text{distill}} + \Delta \mathbf{f}_{\text{Align}}
\end{equation}
\end{proposition}

We provide a local geometric interpretation to explain why these two updates can be empirically compatible in our setting:

\begin{proposition}[Normal-Tangent Approximate Decoupling] \label{prop:decoupling}
Under the assumptions of infinitesimal perturbations and the absence of reward hacking, the dominant component of the distillation output increment ($\Delta \mathbf{f}_{\text{distill}}$) is approximately orthogonal to the true data manifold $\mathcal{M}$. Conversely, the alignment increment ($\Delta \mathbf{f}_{\text{Align}}$) remains approximately tangent to the manifold:
\begin{equation}
\langle \Delta \mathbf{f}_{\text{distill}}, \Delta \mathbf{f}_{\text{Align}} \rangle \approx 0
\end{equation}
\end{proposition}

\textit{Physical Intuition:} The distillation operator acts as a shortest-path projection, pulling off-manifold states back onto $\mathcal{M}$, thus its effect is dominated by a \textbf{normal space} ($N\mathcal{M}$). In contrast, the alignment process (SFT and RL) redistributes probability density along the manifold surface to satisfy instructions and maximize rewards, primarily operating within the \textbf{tangent space} ($T_{\mathbf{x}}\mathcal{M}$). This normal-tangent intuition motivates the approximate decoupling described in Proposition \ref{prop:decoupling}. (See Appendix \ref{appendix:proof_prop1} and \ref{appendix:proof_prop2} for the corresponding local analysis).

\paragraph{Method Implementation}
Guided by this theoretical decoupling, we introduce \textbf{$\Delta$-Space Weight Merge (DSWM)}. Taking the base model as an anchor, we directly sum the distilled checkpoint increments and our alignment increments:
\begin{equation}
\mathbf{W}_{\text{fused}} = \mathbf{W}_{\text{base}} + \Delta \mathbf{W}_{\text{distill}} + \Delta \mathbf{W}_{\text{Align}}
\end{equation}
By deploying this single $\mathbf{W}_{\text{fused}}$ checkpoint, the framework integrates the truncation-error reduction of step distillation (via the normal pull) with the complex reasoning capabilities of closed-loop alignment (via tangent exploration). This offline mechanism circumvents the CoT data reconstruction bottleneck, enabling high-quality, low-latency reasoning inference.

\begin{figure}[t]
  \centering
  \includegraphics[width=0.85\linewidth]{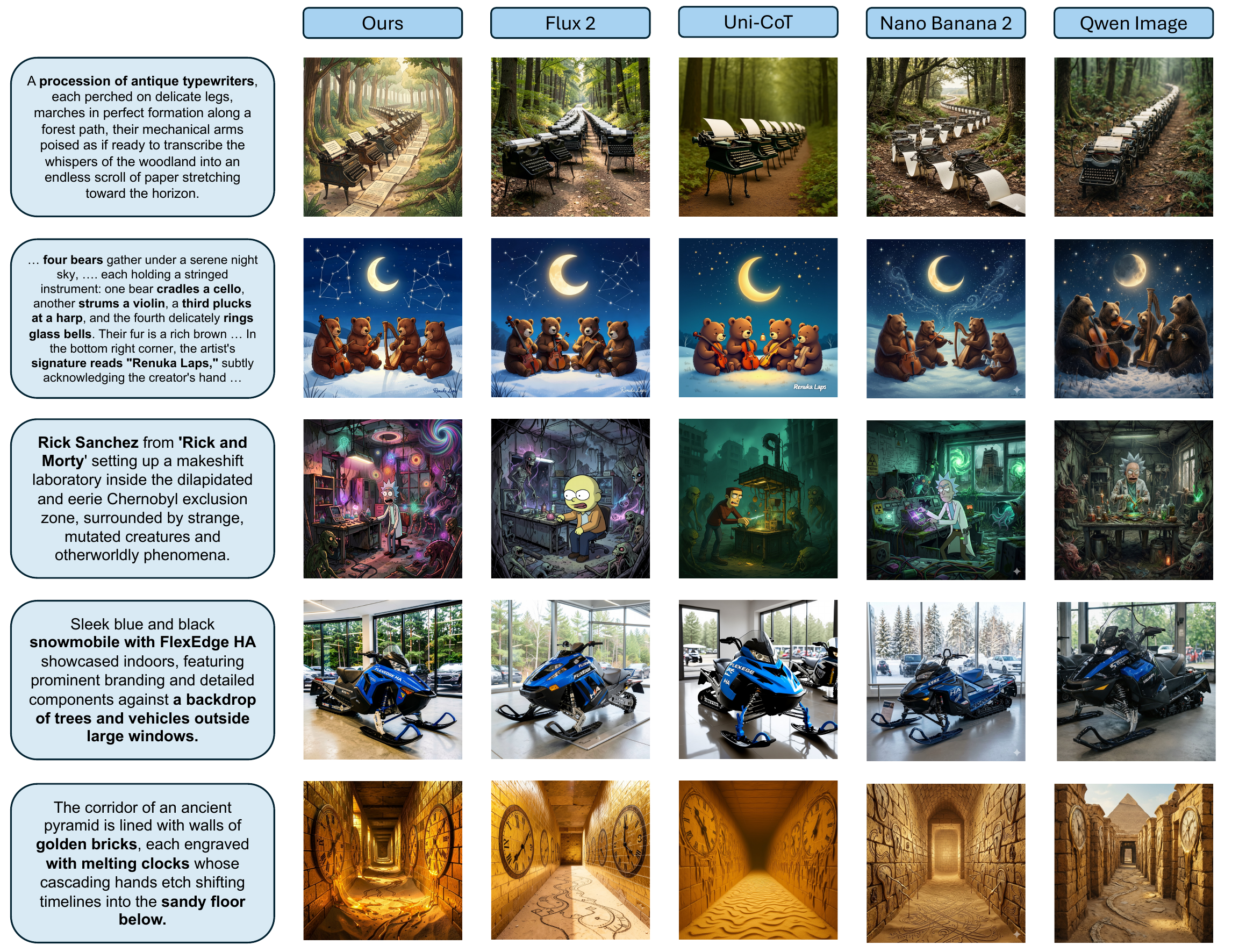}
  \caption{Visual comparison of generation results (CLVR) with other methods. Key control signals in the prompts are highlighted in bold.}
  \label{fig:compare_w_baseline}
\end{figure}

\begin{table}[t]
  \centering
  \small
  \begin{minipage}[t]{0.5\linewidth}
    \vspace{0pt}
    \centering
    \captionof{table}{Quantitative comparisons on GenEval.}
    \label{tab:exp-geneval-detailed}
    \resizebox{0.87\linewidth}{!}{
    \begin{tabular}{lccccccc}
      \toprule
      \begin{tabular}[c]{@{}c@{}}Model\end{tabular} &
      \begin{tabular}[c]{@{}c@{}}Single\\Object\end{tabular} &
      \begin{tabular}[c]{@{}c@{}}Two\\Objects\end{tabular} &
      \begin{tabular}[c]{@{}c@{}}Counting\end{tabular} &
      \begin{tabular}[c]{@{}c@{}}Colors\end{tabular} &
      \begin{tabular}[c]{@{}c@{}}Position\end{tabular} &
      \begin{tabular}[c]{@{}c@{}}Color\\Attributes\end{tabular} &
      \begin{tabular}[c]{@{}c@{}}Overall\end{tabular} \\
    \midrule
    SD3.5-Large \cite{SD3.5} & 0.98 & 0.89 & 0.73 & 0.83 & 0.34 & 0.47 & 0.71 \\
    FLUX.1-dev \cite{flux2024} & 0.98 & 0.93 & 0.75 & 0.93 & 0.68 & 0.65 & 0.82 \\
    Emu3-Gen \cite{Emu3} & 0.99 & 0.81 & 0.42 & 0.80 & 0.49 & 0.45 & 0.66 \\
    Janus-Pro-7B \cite{Janus-Pro} & 0.99 & 0.89 & 0.59 & 0.90 & 0.79 & 0.66 & 0.80 \\
    Show-o2 \cite{Show-o2} & 1.00 & 0.87 & 0.58 & 0.92 & 0.52 & 0.62 & 0.76 \\
    BAGEL-7B \cite{Bagel} & 0.99 & 0.95 & 0.76 & 0.87 & 0.51 & 0.56 & 0.77 \\
    Z-Image \cite{Z-Image}& 1.00 & 0.94 & 0.78 & 0.93 & 0.62 & 0.77 & 0.84 \\
    Qwen-Image \cite{QwenImage}& 0.99 & 0.92 & 0.89 & 0.88 & 0.76 & 0.77 & 0.87 \\
    Seedream 3.0 \cite{Seedream3} & 0.99 & 0.96 & 0.91 & 0.93 & 0.47 & 0.80 & 0.84 \\
    GPT-4o \cite{GPT-4o}& 0.99 & 0.92 & 0.85 & 0.92 & 0.75 & 0.61 & 0.84 \\
    Uni-CoT \cite{UniCoT}& 0.99 & 0.96 & 0.84 & 0.92 & 0.57 & 0.71 & 0.83 \\
    Process-Driven \cite{Process_Driven}  & 0.99 & 0.95 & 0.75 & 0.87 & 0.72 & 0.69 & 0.83 \\
    \midrule
    FLUX.2 4B [Klein] base \cite{flux-2} & 0.99 & 0.87 & 0.62 & 0.85 & 0.52 & 0.59 & 0.74 \\
    CLVR (4B) & 0.99 & 0.92 & 0.85 & 0.89 & 0.85 & 0.71 & 0.87 \\
    \midrule
    FLUX.2 9B [Klein] base \cite{flux-2} & 1.00 & 0.85 & 0.80 & 0.89 & 0.59 & 0.69 & 0.80 \\
    CLVR (9B) & 1.00 & 0.96 & 0.89 & 0.91 & 0.80 & 0.70 & \textbf{0.88} \\
    \bottomrule
    \end{tabular}
    }
  \end{minipage}\hfill
  \begin{minipage}[t]{0.5\linewidth}
    \vspace{0pt}
    \centering
    \captionof{table}{Quantitative comparisons on WiseBench.}
    \label{tab:exp-wisebench-detailed}
    \resizebox{1.06\linewidth}{!}{
    \begin{tabular}{lccccccc}
    \toprule
    Model & Cultural & Time & Space & Biology & Physics & Chemistry & Overall \\
    \midrule
    FLUX.1-dev \cite{flux2024} & 0.48 & 0.58 & 0.62 & 0.42 & 0.51 & 0.35 & 0.50 \\
    SD3.5-Large \cite{SD3.5} & 0.44 & 0.50 & 0.58 & 0.44 & 0.52 & 0.31 & 0.46 \\
    GPT-4o \cite{GPT-4o} & 0.81 & 0.71 & 0.89 & 0.83 & 0.79 & 0.74 & \textbf{0.80} \\
    LongCat-Image \cite{longcat} & 0.66 & 0.61 & 0.72 & 0.66 & 0.72 & 0.49 & 0.65 \\
    Qwen-Image \cite{QwenImage} & 0.62 & 0.63 & 0.77 & 0.57 & 0.75 & 0.40 & 0.62 \\
    Hunyuan-Image 3.0 \cite{HunyuanImage} & 0.58 & 0.57 & 0.70 & 0.56 & 0.63 & 0.31 & 0.57 \\
    BAGEL \cite{Bagel} & 0.44 & 0.55 & 0.68 & 0.44 & 0.60 & 0.39 & 0.52 \\
    T2I-R1 \cite{T2I-R1} & 0.56 & 0.55 & 0.63 & 0.54 & 0.55 & 0.30 & 0.54 \\
    Uni-CoT \cite{UniCoT} & 0.66 & 0.60 & 0.79 & 0.64 & 0.67 & 0.60 & 0.65 \\
    \midrule
    FLUX.2 4B [Klein] base \cite{flux-2} & 0.37 & 0.46 & 0.60 & 0.41 & 0.55 & 0.36 & 0.44 \\
    CLVR (4B) & 0.73 & 0.72 & 0.83 & 0.73 & 0.80 & 0.64 & 0.74 \\
    \midrule
    FLUX.2 9B [Klein] base \cite{flux-2} & 0.46 & 0.56 & 0.66 & 0.50 & 0.62 & 0.42 & 0.52 \\
    CLVR (9B) & 0.76 & 0.73 & 0.83 & 0.76 & 0.79 & 0.67 & 0.76 \\
    \bottomrule
    \end{tabular}
    }
  \end{minipage}
\end{table}

  \begin{table}[t]
    \centering
    \small
    \caption{Quantitative comparisons on PRISM \cite{PRISM-Bench}.}
    \label{tab:exp-prism-detailed}
    \resizebox{\linewidth}{!}{
    \begin{tabular}{lcccccccc}
    \toprule
    Model & Imagination & Entity & Text Rendering & Style & Affection & Composition & Long Text & Overall \\
    \midrule
    Seedream 3.0 \cite{Seedream3} & 76.9 & 77.0 & 63.2 & 85.7 & 89.8 & 89.8 & \underline{75.0} & 79.6 \\
    Qwen-Image \cite{QwenImage} & 79.6 & 76.3 & 61.6 & 86.6 & \underline{90.4} & 90.3 & 74.5 & 79.9 \\
    SD3.5-Large \cite{SD3.5} & 72.3 & 74.3 & 58.9 & 80.7 & 86.2 & 85.9 & 58.0 & 73.9 \\
    FLUX.1-dev \cite{flux2024} & 71.1 & 71.0 & 56.3 & 76.4 & 89.7 & 86.8 & 64.6 & 73.9 \\
    SD3.5-Medium \cite{SD3.5} & 71.3 & 68.3 & 41.7 & 78.9 & 85.2 & 84.3 & 56.9 & 69.5\\
    Uni-CoT \cite{UniCoT} & 74.4 & 55.5 & 37.6 & 73.7 & 80.9 & 80.1 & 60.0 & 66.1 \\
    Janus-Pro-7B \cite{Janus-Pro} & 68.1 & 59.5 & 26.1 & 72.6 & 75.4 & 72.4 & 51.1 & 60.7 \\
    BAGEL \cite{Bagel} & 68.7 & 54.6 & 37.4 & 69.6 & 81.6 & 81.8 & 61.7 & 65.1 \\
    GPT-4o \cite{GPT-4o} & 86.4 & \textbf{88.2} & \underline{69.7}	 & \textbf{90.7} & \textbf{92.1} & \underline{92.8} & 78.3 & \textbf{86.3} \\
    Gemini 2.5-Flash-Image \cite{Gemini25} & \underline{88.6} & \underline{84.2} & \textbf{78.5} & \underline{88.6} & 84.5 & 90.5 & \textbf{81.1} & \underline{85.3} \\
    
    \midrule
    FLUX.2 4B [Klein] base \cite{flux-2} & 71.8 & 53.1 & 45.5 & 69.5 & 74.1 & 79.2 & 61.6 & 65.7 \\
    CLVR (4B) & 87.7 & 64.3 & 54.2 & 83.8 & 86.0 & 88.3 & 69.9 & 76.3 \\
    \midrule
    FLUX.2 9B [Klein] base \cite{flux-2} & 77.2 & 64.8 & 55.5 & 76.7 & 78.3 & 86.4 & 70.0 & 72.7 \\
    CLVR (9B) & \textbf{89.3} & 73.4 & 67.6 & 87.0 & 88.2 & \textbf{94.0} & \underline{75.0} & 82.1 \\
    \bottomrule
    \end{tabular}
    }
  \end{table}

  \begin{table*}[t]
    \centering
    \small
    \caption{Ablation study for proposed techniques. "rewrite" refers to the prompt rewrite operation by Qwen3-VL 8B. }
    \label{tab:exp-ablation-geneval-wisebench}
    
    \text{Panel A: ablation on GenEval}
    \resizebox{\textwidth}{!}{
    \begin{tabular}{lccccccc}
    \toprule
    \begin{tabular}[c]{@{}c@{}}Setting\end{tabular} &
      \begin{tabular}[c]{@{}c@{}}Single\\Object\end{tabular} &
      \begin{tabular}[c]{@{}c@{}}Two\\Objects\end{tabular} &
      \begin{tabular}[c]{@{}c@{}}Counting\end{tabular} &
      \begin{tabular}[c]{@{}c@{}}Colors\end{tabular} &
      \begin{tabular}[c]{@{}c@{}}Position\end{tabular} &
      \begin{tabular}[c]{@{}c@{}}Color\\Attributes\end{tabular} &
      \begin{tabular}[c]{@{}c@{}}Overall\end{tabular} \\
    \midrule
    FLUX.2 4B Distill & 0.99 & 0.91 & 0.77 & 0.86 & 0.69 & 0.64 & 0.81 \\
    FLUX.2 4B Distill rewrite & 0.99 & 0.86 & 0.66 & 0.83 & 0.64 & 0.68 & 0.78 \\
    FLUX.2 4B + CLVR +PPRL +DSWM & 0.99 & 0.92 & 0.85 & 0.89 & 0.85 & 0.71 & 0.87 \\
    \midrule
    FLUX.2 4B base & 0.99 & 0.87 & 0.62 & 0.85 & 0.52 & 0.59 & 0.74 \\
    FLUX.2 4B base + CLVR (w/o PPRL) & 0.96 & 0.88 & 0.76 & 0.81 & 0.71 & 0.57 & 0.78 \\
    FLUX.2 4B base + CLVR (w PPRL) & 0.95 & 0.89 & 0.75 & 0.90 & 0.74 & 0.72 & 0.83 \\
    FLUX.2 4B DSWM + CLVR (w PPRL) & 0.99 & 0.92 & 0.85 & 0.89 & 0.85 & 0.71 & 0.87 \\
    \bottomrule
    \end{tabular}
    }
    
    \vspace{8pt}
    \text{Panel B: ablation on WiseBench.}
    \resizebox{\textwidth}{!}{
    \begin{tabular}{lccccccc}
    \toprule
    Setting & Cultural & Time & Space & Biology & Physics & Chemistry & Overall \\
    \midrule
    FLUX.2 4B Distill & 0.42 & 0.52 & 0.64 & 0.47 & 0.57 & 0.41 & 0.48 \\
    FLUX.2 4B Distill rewrite & 0.64 & 0.60 & 0.75 & 0.61 & 0.67 & 0.55 & 0.64 \\
    FLUX.2 4B DSWM + CLVR (w PPRL) & 0.68 & 0.69 & 0.82 & 0.71 & 0.78 & 0.60 & 0.71 \\
    \bottomrule
    \end{tabular}
    }
  
    \end{table*}
    
\section{Experiment}

\subsection{Implementation details}

In our experimental setup, the VLM controller of CLVR is fixed to use the Qwen3-VL 8B model \cite{Qwen3-VL}, while the diffusion model employs the FLUX.2 Klein 4B and 9B models \cite{flux-2}. During the supervised fine-tuning (SFT) stage, both the diffusion model and the VLM are fully fine-tuned. In contrast, for the reinforcement learning (RL) stage, we employ LoRA fine-tuning for stability. For the base models, the sampling steps are fixed at 28, with a classifier-free guidance (CFG) scale of 4, using the Euler sampler. For distilled models and models utilizing DSWM, we use 4 sampling steps without CFG. Detailed settings are provided in the Appendix \ref{sec:appendix_inference_protocol_en}.

\subsection{Main results on standard T2I benchmarks}

We evaluate our method on five comprehensive benchmarks: GenEval \cite{geneval}, GenEval++ \cite{Echo-4o}, ImagineBench \cite{Echo-4o}, PRISM \cite{PRISM-Bench}, and WiseBench \cite{niu2025wise}. We compare the CLVR method against a wide spectrum of open-source models and unified multimodal models (e.g., SD3.5 \cite{SD3.5}, T2I-R1 \cite{T2I-R1}, Uni-CoT \cite{UniCoT}). Proprietary models like GPT-4o \cite{GPT-4o} and Gemini 2.5 \cite{Gemini25} are included as upper-bound references. For ImagineBench and GenEval++, due to space constraints, we present the detailed results in Appendix \ref{sec:appendix_additional_benchmarks}.

As shown in Tables \ref{tab:exp-geneval-detailed} and \ref{tab:exp-genevalpp-imagine-detailed}, our CLVR (9B) substantially outperforms the FLUX.2 baseline. Notably, on GenEval, CLVR explicitly surpasses recent reasoning-enhanced methods like Uni-CoT and T2I-R1, with notable improvements in complex compositional categories (e.g., spatial positioning, counting, and multi-object generation).

On ImagineBench and PRISM (Table \ref{tab:exp-genevalpp-imagine-detailed} and Table \ref{tab:exp-prism-detailed}), CLVR (9B) reaches overall scores of 8.830 and 82.1 respectively. On PRISM, it outperforms the strongest open-source baseline in our comparison (Qwen-Image, 79.9) by 2.2 points while narrowing the gap to GPT-4o (86.3). Furthermore, on WiseBench (Table \ref{tab:exp-wisebench-detailed}), which emphasizes broad knowledge-grounded generation, our model achieves 0.76, closely approaching the GPT-4o upper bound (0.80).

\subsection{Empirical capacity ceiling of single-step generation}
\label{sec:probe}

We hypothesize that single-step generation paradigms face an inherent performance ceiling on complex semantics, bounded by model capacity. To break this ceiling without simply scaling up the model, we introduce CLVR. To empirically validate this, we design a diagnostic Semantic Complexity Scaling Probe. Further experimental details can be found in the Appendix \ref{sec:probe_appendix}. The probe stratifies prompts into 10 complexity tiers ($C_{\text{task}}$) based on entities, relations, and hard constraints. We evaluate performance using the Area Under the Pass-Complexity Curve ($\text{AUC}_{\text{pass}}$) via LLM-as-a-judge. To correlate performance with capacity, we compute a spectral capacity proxy ($I_{\text{eff}}$) using SVD on core weight layers \cite{I_eff}, which better reflects effective feature space than raw parameter counts.

According to the results in Figure \ref{fig:motivation_exp}, as $C_{\text{task}}$ increases, single-step baselines degrade sharply, requiring exponential $I_{\text{eff}}$ scaling for marginal gains. In contrast, CLVR maintains a resilient pass rate across high-complexity tiers. Compared to FLUX.2, CLVR improves $\text{AUC}_{\text{pass}}$, mitigating the structural capacity ceiling without expanding the DiT backbone.

\begin{figure}[t]
  \centering
  \includegraphics[width=\linewidth]{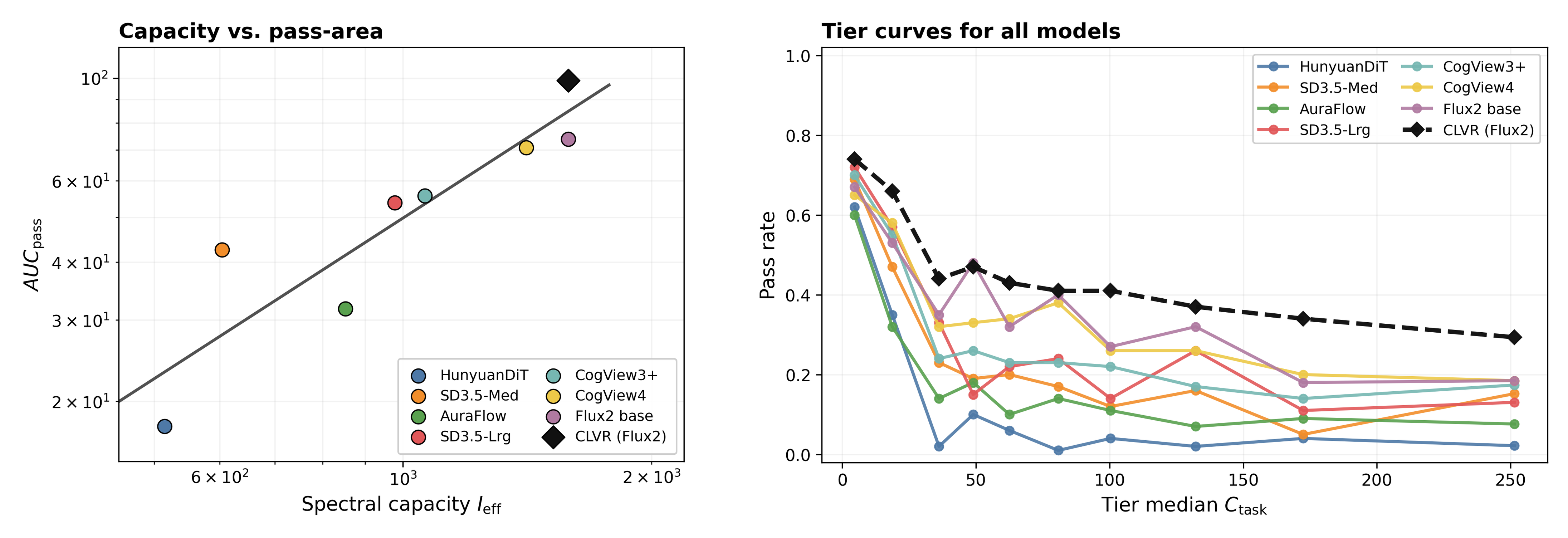}
  \caption{Quantitative results of the semantic complexity probe. $C_{\text{task}}$ stratifies prompt difficulty (entities, relations, and hard constraints); $I_{\text{eff}}$ is an entropy effective-rank spectral proxy for backbone capacity~\cite{I_eff}. The left plot shows that CLVR achieves a higher Area Under the Pass-Complexity Curve (AUC Pass) compared to single-step models of similar capacity. The right plot illustrates that while single-step models' performance drops sharply as task complexity increases, CLVR maintains a high pass rate across all tiers.}
  \label{fig:motivation_exp}
\end{figure}

\subsection{Ablation studies}

\paragraph{Prompt Rewrite vs. CLVR}

A natural question is whether the performance gains merely stem from the VLM rewriting the prompt into a more descriptive format, leveraging its inherent superior semantic understanding. In Table \ref{tab:exp-ablation-geneval-wisebench} (Panel A, B), we compare the FLUX.2 4B Distill baseline with an open-loop variant that uses Qwen3-VL 8B for prompt rewriting. 

Open-loop prompt rewriting improves WiseBench scores (0.48 to 0.64) by providing external knowledge, but degrades GenEval scores (0.81 to 0.78) where instructions are already explicit. In contrast, CLVR achieves superior scores on both benchmarks (0.87 on GenEval, 0.74 on WiseBench). This confirms that CLVR's gains extend beyond the VLM's semantic enrichment. Instead, the improvements stem fundamentally from (1) task decomposition, which avoids forcing the model into one-shot generation and raises its capacity ceiling, and (2) iterative visual self-correction, which ensures each generative step remains within the diffusion model's reliable execution boundaries. This explains why our method comprehensively outperforms open-loop rewriting.

\paragraph{Training Alignment and Deployment Acceleration}

We ablate the training and deployment mechanisms to validate their individual contributions (Table \ref{tab:exp-ablation-geneval-wisebench}, Panel A). Starting from the FLUX.2 4B base (0.74 on GenEval), applying CLVR via SFT alone improves the score to 0.78. Further applying PPRL boosts the score to 0.83, validating that proxy prompts effectively translate long-context histories into explicit reward signals for stable alignment.

Finally, employing DSWM yields a GenEval score of 0.87, surpassing both the RL-only (0.83) and distill-only (0.81) baselines. This confirms that alignment and distillation priors merge compatibly without destructive interference. By reusing the distillation prior, DSWM reduces iterative denoising from $28 \times 2$ to just 4 NFEs (number of function evaluations) per step, effectively resolving the computational bottleneck of closed-loop reasoning. We also provide an end-to-end accelerate analysis in Appendix \ref{sec:appendix_inference_acceleration_en}.

\paragraph{Qualitative Cases and Visual Comparisons}

To intuitively illustrate the mechanics of our framework, we provide detailed real-world Closed-Loop Visual Reasoning (CLVR) trajectories in the Appendix, Figure \ref{fig:trace_showcase}, showcasing how a complex prompt is systematically executed across multiple reasoning and generation steps. Furthermore, Figure \ref{fig:compare_w_baseline} presents qualitative comparisons between our method (CLVR 4B) and strong baselines (e.g., FLUX.2 4B, Uni-CoT, Nano Banana 2 \cite{NanoBanana2}, and Qwen-Image) on challenging prompts. The results show that our framework approaches state-of-the-art proprietary instruction following and improves over open-source baselines.

\section{Conclusion}

To address the capability degradation of single-step generation and the diminishing returns of parameter scaling when text-to-image (T2I) models handle complex semantics, we propose Closed-Loop Visual Reasoning (CLVR). We construct a system-level solution across four dimensions: data, training, inference, and deployment. (1) Reliable Data Synthesis: We eliminate ungrounded planning hallucinations via a state-constrained controller and step-level visual verification, guaranteeing high-quality training trajectories. (2) Long-Context Alignment: We propose Proxy Prompt Reinforcement Learning (PPRL) to distill interleaved image-text histories into explicit single-step reward signals, overcoming optimization bottlenecks in long-horizon planning. (3) Globally Consistent Inference: We introduce trajectory-accumulative conditioning to retain multimodal historical memory, effectively mitigating the loss of long-horizon dependencies. (4) Efficient Model Deployment: We propose $\Delta$-Space Weight Merge (DSWM) to decouple CLVR training from distillation acceleration, reducing the denoising cost to 4 NFEs and eliminating the need for expensive end-to-end re-distillation. Extensive experiments demonstrate that CLVR achieves significant improvements across multiple benchmarks.



{\small
\bibliographystyle{plainnat}
\bibliography{references}
}


\appendix

\section{Technical appendices and supplementary material}

\subsection{Local Analysis and Geometric Motivation for Weight Merge} \label{sec:appendix_weight_merge}

This section provides a local first-order analysis and geometric motivation for $\Delta$-Space Weight Merge (DSWM) mechanism introduced in the main text. Throughout this section, $\Delta \mathbf{f}$ denotes the first-order change in the model output for a fixed noisy state $(\mathbf{x}_t,t)$ induced by a parameter perturbation around $\mathbf{W}_{\text{base}}$. The following proofs demonstrate why the weight increments from distillation ($\Delta \mathbf{W}_{\text{distill}}$) and alignment fine-tuning ($\Delta \mathbf{W}_{\text{Align}}$) can be linearly added in the parameter space while maintaining approximate geometric decoupling in the output space. Following the definition in the main text, we define the alignment increment as $\Delta \mathbf{W}_{\text{Align}} = \Delta \mathbf{W}_{\text{SFT}} + \Delta \mathbf{W}_{\text{RL}}$, treating both supervised fine-tuning (SFT) and reinforcement learning (RL) as unified alignment operations.

\textbf{Proof Outline:} We establish the theoretical validity of DSWM by demonstrating that distillation and alignment updates occur in approximately orthogonal subspaces (normal and tangent spaces, respectively). First, we assume local linear superposition, allowing weight increments to be treated additively. Second, we demonstrate that distillation paradigms introduce parameter updates along the \textit{normal space} of the data manifold. In contrast, alignment fine-tuning (particularly RL) introduces updates primarily within the \textit{tangent space}. This geometric view helps explain why the two updates can be merged with limited interference in our evaluated setting, as further supported by the ablation results.

\subsubsection{Linear Superposition under First-Order Perturbation}
\label{appendix:proof_prop1}

\begin{hypothesis}[Local Linear Perturbation] \label{hyp:linear}
Let the prediction function of the base diffusion model be $f(\mathbf{x}_t, t; \mathbf{W})$. We assume that $f$ is at least twice continuously differentiable with respect to its parameters $\mathbf{W}$ in the neighborhood of the base weights $\mathbf{W}_{\text{base}}$. We further assume that the magnitudes of the parameter increments induced by distillation ($\Delta \mathbf{W}_{\text{distill}}$) and alignment ($\Delta \mathbf{W}_{\text{Align}}$) are sufficiently small such that they lie within the local linear perturbation region, allowing us to safely truncate the second-order Taylor remainder $\mathcal{O}(\|\Delta \mathbf{W}\|^2)$.
\end{hypothesis}

\textbf{Empirical Evidence for Hypothesis~\ref{hyp:linear}:} To validate the "local small perturbation" assumption, we measured parameter space deviations on the FLUX.2 4B model series. We adopt the global relative Frobenius norm to measure the weight shift. Specifically, for a model comprising $K$ tensors, the deviation between a reference model with weights $\mathbf{W}^{(1)}$ and a perturbed model with weights $\mathbf{W}^{(2)}$ (where the weight difference is denoted as $\Delta \mathbf{W} = \mathbf{W}^{(2)} - \mathbf{W}^{(1)}$) is defined as:

\begin{equation}
R(\mathbf{W}^{(1)}, \mathbf{W}^{(2)}) = \|\Delta \mathbf{W}\|_F / \|\mathbf{W}^{(1)}\|_F = \sqrt{\sum_{k=1}^K \|\Delta \mathbf{W}_k\|_F^2} / \sqrt{\sum_{k=1}^K \|\mathbf{W}_k^{(1)}\|_F^2}, 
\end{equation}

which measures the ratio of the deviation to the total parameter energy of the reference model. Empirical results demonstrate that the global relative Frobenius norm shift induced by the distilled model relative to $\mathbf{W}_{\text{base}}$ is approximately $2.79\%$; the shift caused by full-parameter SFT is approximately $2.30\%$; and the effective weight increment $\Delta \mathbf{W}_{\text{RL}} = \frac{\alpha}{r}\mathbf{B}\mathbf{A}$ induced by RL LoRA fine-tuning on top of the SFT base results in an even lower relative shift of approximately $0.0075\%$. These solid quantitative engineering data indicate that the model weight updates induced by both distillation and the entire alignment phase (SFT and RL) indeed constitute only a minimal local shift in the parameter space. This provides sufficient empirical support for safely truncating the higher-order remainder $\mathcal{O}(\|\Delta \mathbf{W}\|^2)$ in our subsequent derivations via Taylor expansion.

\begin{proposition}[Additive Approximation of Output Increments] \label{prop:additive}
Under Hypothesis~\ref{hyp:linear}, the output increment of the merged model can be approximately decomposed into a linear superposition of the increments from the two independent tasks:
\begin{equation}
    f(\mathbf{W}_{\text{base}} + \Delta \mathbf{W}_{\text{distill}} + \Delta \mathbf{W}_{\text{Align}}) \approx f(\mathbf{W}_{\text{base}}) + \Delta \mathbf{f}_{\text{distill}} + \Delta \mathbf{f}_{\text{Align}}.
\end{equation}
\end{proposition}

\begin{proof}
Given the base model weights $\mathbf{W}_{\text{base}}$, we perform a first-order Taylor expansion on the output function $f(\mathbf{x}_t, t; \mathbf{W})$ around $\mathbf{W}_{\text{base}}$ for the combined perturbation $\Delta \mathbf{W}_{\text{merge}} = \Delta \mathbf{W}_{\text{distill}} + \Delta \mathbf{W}_{\text{Align}}$:
\begin{equation}
    f(\mathbf{W}_{\text{base}} + \Delta \mathbf{W}_{\text{distill}} + \Delta \mathbf{W}_{\text{Align}}) = f(\mathbf{W}_{\text{base}}) + \mathbf{J}_{\mathbf{W}} (\Delta \mathbf{W}_{\text{distill}} + \Delta \mathbf{W}_{\text{Align}}) + \mathcal{O}(\|\Delta \mathbf{W}\|^2),
\end{equation}
where $\mathbf{J}_{\mathbf{W}} = \nabla_{\mathbf{W}} f$ is the Jacobian matrix evaluated at $\mathbf{W}_{\text{base}}$. We define the approximate output increments for each independent fine-tuning task as:
\begin{equation}
    \Delta \mathbf{f}_{\text{distill}} \triangleq \mathbf{J}_{\mathbf{W}} \Delta \mathbf{W}_{\text{distill}}, \quad \Delta \mathbf{f}_{\text{Align}} \triangleq \mathbf{J}_{\mathbf{W}} \Delta \mathbf{W}_{\text{Align}}.
\end{equation}
According to the linearity of matrix multiplication, we have $\mathbf{J}_{\mathbf{W}} (\Delta \mathbf{W}_{\text{distill}} + \Delta \mathbf{W}_{\text{Align}}) = \mathbf{J}_{\mathbf{W}} \Delta \mathbf{W}_{\text{distill}} + \mathbf{J}_{\mathbf{W}} \Delta \mathbf{W}_{\text{Align}}$. Based on Hypothesis~\ref{hyp:linear}, we truncate the higher-order infinitesimal term $\mathcal{O}(\|\Delta \mathbf{W}\|^2)$ to obtain:
\begin{equation}
    f(\mathbf{W}_{\text{base}} + \Delta \mathbf{W}_{\text{merge}}) \approx f(\mathbf{W}_{\text{base}}) + \Delta \mathbf{f}_{\text{distill}} + \Delta \mathbf{f}_{\text{Align}}.
\end{equation}
This concludes the proof.
\end{proof}

\subsubsection{Compatibility for Distribution-based Distillation}
\label{appendix:proof_prop2}

In this section, we attempt to prove the merge compatibility for distribution-based distillation method such as \cite{ADD, DMD2, DMD}.

\begin{hypothesis}[Manifold Geometry and Healthy Optimization Process] \label{hyp:manifold}
We assume that the true data distribution primarily resides on a low-dimensional manifold $\mathcal{M} \subset \mathbb{R}^d$. For the alignment fine-tuning phase, we assume the optimization process is healthy (i.e., devoid of reward hacking) and is regularized by the Kullback-Leibler (KL) divergence. This implies that the alignment optimization only redistributes the probability density along the surface of the manifold $\mathcal{M}$, without generating destructive noise or meaningless data that deviates from the manifold.
\end{hypothesis}

\begin{proposition}[Normal-Tangent Approximate Decoupling] \label{prop:normal_tangent_appendix}
Under Hypotheses~\ref{hyp:linear} and \ref{hyp:manifold}, the output increment from distribution-matching distillation ($\Delta \mathbf{f}_{\text{distill}}$) essentially acts as a shortest-path projection operator that pulls deviated states back to the manifold, whose dominant component lies in the normal space $N\mathcal{M}$. Conversely, the alignment increment ($\Delta \mathbf{f}_{\text{Align}}$) is expected to be concentrated near the tangent space $T\mathcal{M}$ of the manifold. Therefore, they are approximately orthogonal in the manifold geometry:
\begin{equation}
    \langle \Delta \mathbf{f}_{\text{distill}}, \Delta \mathbf{f}_{\text{Align}} \rangle \approx 0, \quad \forall \mathbf{x} \in \mathcal{M}.
\end{equation}
\end{proposition}

\begin{proof}
\textbf{(1) Proof of Normal Space Residency ($\Delta \mathbf{f}_{\text{distill}} \in N\mathcal{M}$):}
Distribution-matching distillation (such as Distribution Matching Distillation, DMD) aims to minimize the reverse KL divergence between the generated distribution $p_{\theta}$ and the true distribution $p_{\text{data}}$:
\begin{equation}
    \mathcal{L}_{\text{DMD}} = D_{\text{KL}}(p_{\theta} \| p_{\text{data}}).
\end{equation}
The gradient of this loss with respect to a generated sample $\mathbf{x}$ is given by:
\begin{equation}
    \nabla_{\mathbf{x}} \mathcal{L}_{\text{DMD}} = \nabla_{\mathbf{x}} \log p_{\theta}(\mathbf{x}) - \nabla_{\mathbf{x}} \log p_{\text{data}}(\mathbf{x}).
\end{equation}
In practical training, a stop-gradient operation is typically applied to the true distribution score term. Thus, when a large-step single-step prediction $\hat{\mathbf{x}}_{\text{fast}}$ deviates from the manifold, the optimization direction of the model's output increment $\Delta \mathbf{f}_{\text{distill}}$ is primarily determined by the score of the true distribution, i.e.,
\begin{equation}
    \Delta \mathbf{f}_{\text{distill}} \propto \nabla_{\mathbf{x}} \log p_{\text{data}}^{\sigma}(\hat{\mathbf{x}}_{\text{fast}}).
\end{equation}

Utilizing Tweedie's formula, for an observation
\begin{equation}
    \mathbf{x} = \mathbf{z} + \mathbf{\epsilon}, \quad \text{where } \mathbf{\epsilon} \sim \mathcal{N}(\mathbf{0}, \sigma^2 \mathbf{I}),
\end{equation}
the posterior expectation of the true signal can be expressed by the score of the marginal distribution:
\begin{equation}
    \mathbb{E}[\mathbf{z} | \mathbf{x}] = \mathbf{x} + \sigma^2 \nabla_{\mathbf{x}} \log p^{\sigma}(\mathbf{x}).
\end{equation}
In the large-step generation scenario of single-step distillation, the model attempts to directly predict the clean image $\mathbf{x}_0$ at timestep $t=0$ from any given noisy state $\hat{\mathbf{x}}_{\text{fast}}$ in one step. Since it seeks to recover the original true data with little residual Gaussian corruption, its corresponding target distribution $p_{\text{data}}$ can be viewed through the small-noise limit of the smoothed density $p_{\text{data}}^{\sigma}$. According to the asymptotic theory of denoising autoencoders formulated by \cite{alain2014regularizedautoencoderslearndata}, when $\hat{\mathbf{x}}_{\text{fast}}$ lies in a local tubular neighborhood of the data manifold and $\sigma$ is sufficiently small, the normal component of the score of the smoothed distribution is dominated by the residual that pulls the point back toward the nearest manifold surface. Up to tangential score terms and curvature-dependent lower-order corrections, this yields the local approximation:
\begin{equation}
    \nabla_{\mathbf{x}} \log p_{\text{data}}^{\sigma}(\hat{\mathbf{x}}_{\text{fast}}) \approx \frac{1}{\sigma^2} (\pi_{\mathcal{M}}(\hat{\mathbf{x}}_{\text{fast}}) - \hat{\mathbf{x}}_{\text{fast}}),
\end{equation}
where $\pi_{\mathcal{M}}$ is the nearest Euclidean projection onto $\mathcal{M}$. According to the extremum condition in the calculus of variations, the shortest vector from a point outside the manifold to its nearest projected point must be orthogonal to the tangent space at the projected point. Therefore, $(\pi_{\mathcal{M}}(\hat{\mathbf{x}}_{\text{fast}}) - \hat{\mathbf{x}}_{\text{fast}}) \in N_{\pi_{\mathcal{M}}(\hat{\mathbf{x}}_{\text{fast}})} \mathcal{M}$, which implies that the dominant corrective component of $\Delta \mathbf{f}_{\text{distill}}$ resides in $N\mathcal{M}$.

\textbf{(2) Proof of Tangent Space Residency ($\Delta \mathbf{f}_{\text{Align}} \in T\mathcal{M}$):}
The alignment increment $\Delta \mathbf{W}_{\text{Align}} = \Delta \mathbf{W}_{\text{SFT}} + \Delta \mathbf{W}_{\text{RL}}$ encompasses both SFT and RL. The optimization objective of SFT is to maximize the log-likelihood on a high-quality human instruction data subset:
\begin{equation}
    \mathcal{L}_{\text{SFT}} = \mathbb{E}_{\mathbf{x} \sim p_{\text{data}}^{\text{high}}}[-\log p_{\theta}(\mathbf{x})].
\end{equation}
Meanwhile, the optimization objective of RL is typically to maximize the human preference reward under a strict KL divergence constraint:
\begin{equation}
    \mathcal{L}_{\text{RL}} = \max_{\theta} \mathbb{E}_{\mathbf{x} \sim p_{\theta}}[R(\mathbf{x})] - \beta D_{\text{KL}}(p_{\theta} \| p_{\text{base}}).
\end{equation}

The common characteristic of these two optimizations is that SFT only increases the likelihood in high-density regions of the true data manifold $\mathcal{M}$, while the KL penalty term in RL primarily forces the support set of the fine-tuned distribution not to deviate from the support set of the base model $p_{\text{base}}$ (i.e., the manifold $\mathcal{M}$). Simultaneously, the reward model $R(\mathbf{x})$ typically yields extremely high penalties in noisy regions deviating from the manifold. 

Consequently, driven by these two objective functions, the optimization behavior of the model merely redistributes the probability density along the surface of the manifold $\mathcal{M}$ to discover regions with higher aesthetics or superior reasoning quality. Although the true data manifold generally possesses curvature, under the small perturbation constraint (Hypothesis~\ref{hyp:linear}), the deviation from the manifold curvature caused by this displacement within the local linear approximation constitutes a higher-order infinitesimal $\mathcal{O}(\|\Delta \mathbf{f}_{\text{Align}}\|^2)$, which is negligible. Thus, the increment guided by alignment is primarily encouraged to remain close to the tangent space: $\Delta \mathbf{f}_{\text{Align}} \in T_{\mathbf{x}} \mathcal{M}$.

Because the normal space and the tangent space are orthogonal to each other ($N_{\mathbf{x}}\mathcal{M} \perp T_{\mathbf{x}}\mathcal{M}$), their inner product is zero, yielding $\langle \Delta \mathbf{f}_{\text{distill}}, \Delta \mathbf{f}_{\text{Align}} \rangle \approx 0$. This concludes the proof.
\end{proof}

\subsubsection{Compatibility for Trajectory-based Distillation}

In this section, we analyze the merge compatibility for trajectory-based distillation method, such as \cite{CM, sCM, pgd}.

\begin{hypothesis}[Time-Locality and Slowly Varying Alignment Field] \label{hyp:ode}
Let the probability flow ordinary differential equation (PF-ODE) corresponding to the base diffusion model be:
\begin{equation}
    d\mathbf{x} = \mathbf{V}(\mathbf{x}, t) dt.
\end{equation}
The alignment process introduces an additional vector field $\mathbf{U}(\mathbf{x}_t, t)$ into the base vector field. We assume that the effective action of $\mathbf{U}$ is concentrated within a specific training time window $\mathcal{I}_{\text{Align}} \subset [0, 1]$, and its variation with respect to time $t$ is locally slow within this window:
\begin{equation}
    \int_{\mathcal{I}_{\text{Align}}} \left\| \frac{\partial \mathbf{U}}{\partial t} \right\| dt \le \varepsilon_{\text{Align}}, \quad \mathbf{U}(\mathbf{x}_t, t) \approx \mathbf{0} \text{ for } t \notin \mathcal{I}_{\text{Align}}.
\end{equation}
\end{hypothesis}

\begin{proposition}[Bounded Truncation Error Decoupling of Large-Step Integration] \label{prop:truncation}
Under the premise of Hypothesis~\ref{hyp:ode}, the truncation error brought by the highly non-linear base PF-ODE is largely compensated by the distillation correction $\Delta \mathbf{f}_{\text{distill}}$. Meanwhile, due to the slow temporal variation of the additional vector field $\mathbf{U}$ injected by the alignment process, its continuous time integral can be approximated by a single-step estimate. Consequently, compared to the final image obtained via precise multi-step solving of the pure alignment model, the deviation of the final image generated by the merged model in a single step is primarily controlled by the fluctuation magnitude (temporal variation) of $\mathbf{U}$ within the training time window. This suggests a mechanism for mitigating the cumulative truncation-error effects across multiple sampling steps:
\begin{equation}
    \|\mathbf{x}_{\text{merge}} - \tilde{\mathbf{x}}_0\| = \mathcal{O}(\varepsilon_{\text{Align}}) + \mathcal{O}(\|\Delta \mathbf{W}\|^2).
\end{equation}
\end{proposition}

\begin{proof}
Single-step trajectory distillation learns a solver to eliminate the truncation error, such that:
\begin{equation}
    f(\mathbf{W}_{\text{base}} + \Delta \mathbf{W}_{\text{distill}})(\mathbf{x}_1) \approx \mathbf{x}_1 - \int_{0}^{1} \mathbf{V}(\mathbf{x}(\tau), \tau) d\tau = \mathbf{x}_0^{\text{base}}.
\end{equation}
For the pure alignment model, the generation endpoint under exact multi-step solving is:
\begin{equation}
    \tilde{\mathbf{x}}_0 = \mathbf{x}_1 - \int_{0}^{1} \mathbf{V}(\tilde{\mathbf{x}}(\tau), \tau) d\tau - \int_{0}^{1} \mathbf{U}(\tilde{\mathbf{x}}(\tau), \tau) d\tau.
\end{equation}
According to Hypothesis~\ref{hyp:ode}, the effect of the alignment field outside the time window $\mathcal{I}_{\text{Align}}$ is negligible; hence, the multi-step integral can be approximated as:
\begin{equation}
    \tilde{\mathbf{x}}_0 \approx \mathbf{x}_0^{\text{base}} - \int_{\mathcal{I}_{\text{Align}}} \mathbf{U}(\tilde{\mathbf{x}}(\tau), \tau) d\tau.
\end{equation}

For the directly merged model, utilizing the linear superposition principle established in Proposition~\ref{prop:additive}, its single-step large-step prediction is:
\begin{equation}
    f_{\text{merge}}(\mathbf{x}_1) \approx \mathbf{x}_0^{\text{base}} + \Delta \mathbf{f}_{\text{Align}}.
\end{equation}
Here, $\Delta \mathbf{f}_{\text{Align}}$ can be conceptualized as an approximation to the integral of the alignment field. If we anchor this approximation at a representative timestamp $t^\ast \in \mathcal{I}_{\text{Align}}$ within the training time window, it corresponds to multiplying a local representative value by the interval length:
\begin{equation}
    \Delta \mathbf{f}_{\text{Align}} \approx -|\mathcal{I}_{\text{Align}}| \mathbf{U}(\mathbf{x}_{t^\ast}, t^\ast).
\end{equation}

Given the assumption that the alignment vector field varies slowly over time (local slow variation, Hypothesis~\ref{hyp:ode}), the error introduced by substituting continuous integration with this single-point sampling is bounded primarily by its fluctuation over time, i.e., constrained by $\mathcal{O}(\varepsilon_{\text{Align}})$. Therefore, we have:
\begin{equation}
    \left\| -|\mathcal{I}_{\text{Align}}| \mathbf{U}(\mathbf{x}_{t^\ast}, t^\ast) - \left(- \int_{\mathcal{I}_{\text{Align}}} \mathbf{U}(\tilde{\mathbf{x}}(\tau), \tau) d\tau \right) \right\| \le \mathcal{O}(\varepsilon_{\text{Align}}).
\end{equation}
Combining the local approximations above, the deviation between the single-step output of the merged model and the multi-step output of the pure alignment model can be written as:
\begin{equation}
    \|f_{\text{merge}}(\mathbf{x}_1) - \tilde{\mathbf{x}}_0\| = \mathcal{O}(\varepsilon_{\text{Align}}) + \mathcal{O}(\|\Delta \mathbf{W}\|^2).
\end{equation}
This analysis suggests that combining the alignment vector field with the trajectory distillation operator can approximate the alignment-field integral through a single-step jump under local smoothness assumptions. This helps explain why the merged model can reduce severe truncation-error accumulation from large step sizes in our evaluated setting. This concludes the proof.
\end{proof}

\textbf{Remark (Physical Intuition and Empirical Relation):} The theoretical error bounds established above ($\mathcal{O}(\|\Delta \mathbf{W}\|^2)$ and $\mathcal{O}(\varepsilon_{\text{Align}})$) are fundamentally local approximations. In practical applications, as corroborated by our ablation experiments, $\Delta$-Space Weight Merge is compatible in our evaluated setting, significantly accelerating inference without the prerequisite of constructing expensive closed-loop distillation data. However, the efficacy of this method fundamentally relies on the structural similarity and parameter proximity between each fine-tuned model and the base model, which restricts its applicable boundaries when drastic weight changes occur. Our mathematical framework clearly delineates these boundaries, providing a robust theoretical foundation elucidating why this direct parameter fusion successfully operates during diffusion model deployment without substantial interference.

\begin{table*}[t]
  \centering
  \small
  \caption{Quantitative comparisons on GenEval++ and ImagineBench \cite{Echo-4o}.}
  \label{tab:exp-genevalpp-imagine-detailed}
  \resizebox{\textwidth}{!}{
  \begin{tabular}{lccccccccccccc}
  \toprule
  & \multicolumn{8}{c}{GenEval++} & \multicolumn{5}{c}{ImagineBench} \\
  \cmidrule(lr){2-9}\cmidrule(lr){10-14}
\begin{tabular}[c]{@{}c@{}}Model\end{tabular} &
\begin{tabular}[c]{@{}c@{}}Color\end{tabular} &
\begin{tabular}[c]{@{}c@{}}Count\end{tabular} &
\begin{tabular}[c]{@{}c@{}}Color\\Count\end{tabular} &
\begin{tabular}[c]{@{}c@{}}Color\\Pos\end{tabular} &
\begin{tabular}[c]{@{}c@{}}Pos\\Count\end{tabular} &
\begin{tabular}[c]{@{}c@{}}Pos\\Size\end{tabular} &
\begin{tabular}[c]{@{}c@{}}Multi\\Count\end{tabular} &
\begin{tabular}[c]{@{}c@{}}Overall\end{tabular} &
\begin{tabular}[c]{@{}c@{}}Attribute\\Shift\end{tabular} &
\begin{tabular}[c]{@{}c@{}}Spatio\\Temporal\end{tabular} &
\begin{tabular}[c]{@{}c@{}}Hybridization\end{tabular} &
\begin{tabular}[c]{@{}c@{}}Multi\\Object\end{tabular} &
\begin{tabular}[c]{@{}c@{}}Overall\end{tabular} \\
  \midrule
  FLUX.1-Kontext \cite{flux2024} & 0.425 & 0.500 & 0.200 & 0.250 & 0.300 & 0.400 & 0.325 & 0.343 & 5.330 & 6.490 & 5.480 & 5.340 & 5.620 \\
  FLUX.1-dev \cite{flux2024} & 0.350 & 0.625 & 0.150 & 0.275 & 0.200 & 0.375 & 0.225 & 0.314 & 5.680 & 7.130 & 6.380 & 5.240 & 6.060 \\
  GPT-4o \cite{GPT-4o} & 0.900 & 0.675 & 0.725 & 0.625 & 0.600 & 0.800 & 0.850 & \textbf{0.739} & 8.540 & 9.180 & 8.570 & 7.980 & 8.560 \\
  Janus-Pro \cite{Janus-Pro} & 0.450 & 0.300 & 0.125 & 0.300 & 0.075 & 0.350 & 0.125 & 0.246 & 5.300 & 7.280 & 6.730 & 6.040 & 6.220 \\
  
  OmniGen2 \cite{OmniGen2} & 0.550 & 0.425 & 0.200 & 0.275 & 0.125 & 0.250 & 0.450 & 0.325 & 5.280 & 7.450 & 6.290 & 6.310 & 6.220 \\
  BAGEL \cite{Bagel} & 0.325 & 0.600 & 0.250 & 0.325 & 0.250 & 0.475 & 0.375 & 0.371 & 5.370 & 6.930 & 6.500 & 6.410 & 6.200 \\
  T2I-R1 \cite{T2I-R1} & 0.675 & 0.325 & 0.200 & 0.350 & 0.075 & 0.250 & 0.300 & 0.311 & 5.850 & 7.700 & 7.360 & 6.680 & 6.780 \\
  Uni-CoT \cite{UniCoT} & 0.700 & 0.675 & 0.655 & 0.625 & 0.600 & 0.625 & 0.575 & 0.635 & 7.438 & 8.393 & 7.563 & 7.460 & 7.747 \\

  \midrule
  FLUX.2 4B [Klein] \cite{flux-2}  & 0.750 & 0.425 & 0.225 & 0.300 & 0.050 & 0.450 & 0.425 & 0.375 & 6.342 & 5.877 & 6.543 & 6.953 & 6.267 \\
  CLVR (4B) & 0.900 & 0.450 & 0.450 & 0.650 & 0.600 & 0.725 & 0.538 & 0.616 & 7.921 & 8.977 & 8.664 & 8.137 & 8.435 \\
  \midrule
  FLUX.2 9B [Klein] \cite{flux-2} & 0.750 & 0.450 & 0.250 & 0.000 & 0.050 & 0.275 & 0.375 & 0.307 & 6.669 & 8.097 & 7.360 & 7.130 & 7.274 \\
  CLVR (9B) & 0.875 & 0.625 & 0.600 & 0.750 & 0.575 & 0.700 & 0.700 & 0.689 & 8.357 & 9.537 & 8.833 & 8.750 & \textbf{8.830} \\
  \bottomrule
  \end{tabular}
  }
  \end{table*}

  \begin{figure}[t]
    \centering
    \includegraphics[width=\linewidth]{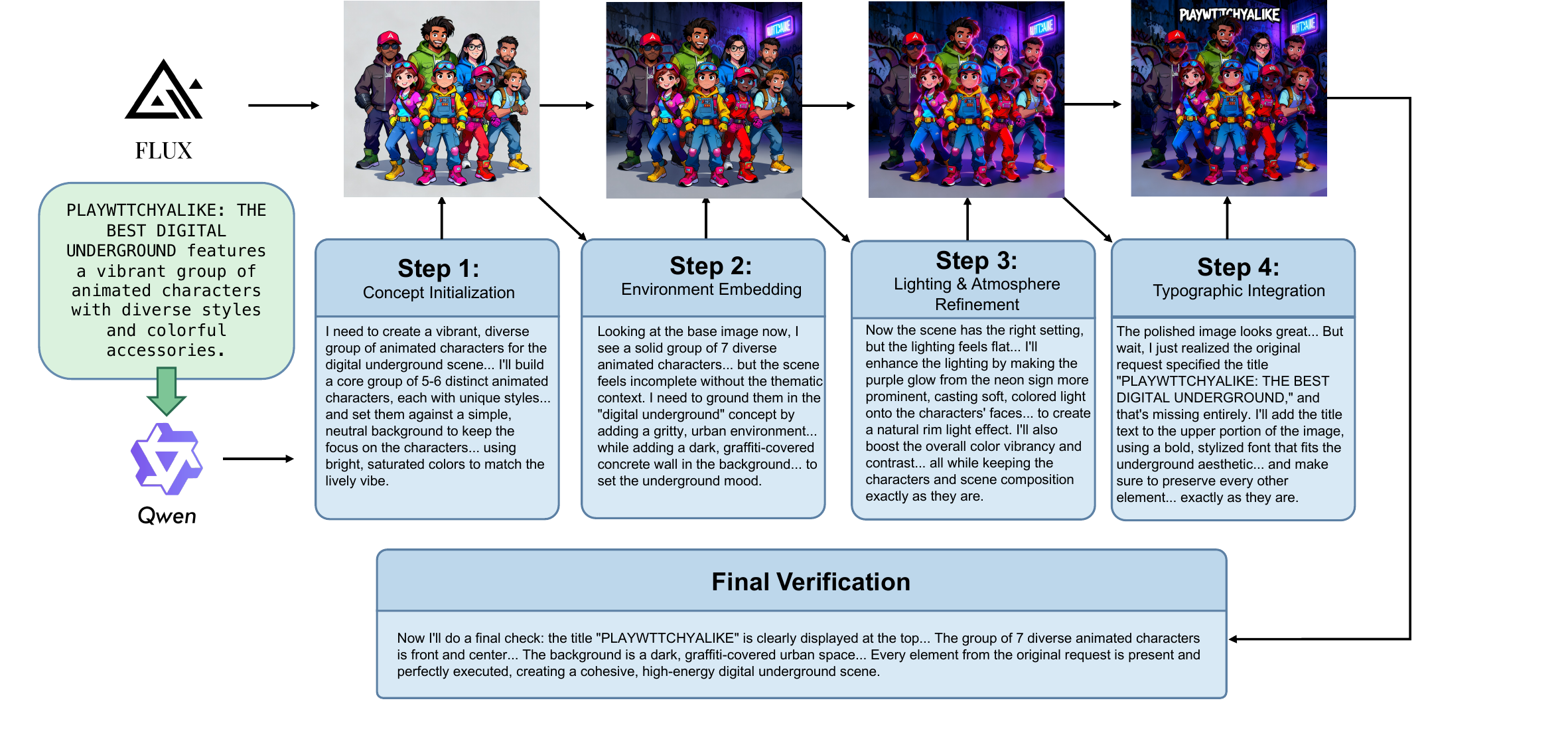}
    \caption{A step-by-step CLVR inference case. The trajectory begins with concept initialization (Step 1), followed by environment embedding (Step 2), lighting and atmosphere refinement (Step 3), and final typographic integration (Step 4), resulting in a cohesive image that fulfills all complex prompt requirements.}
    \label{fig:cot_case}
  \end{figure}

\subsection{Benchmark Descriptions}
\label{sec:appendix_benchmark_descriptions}

In this section, we provide a detailed introduction to the five comprehensive benchmarks used in our evaluation, covering aspects from fine-grained compositional alignment to complex world-knowledge reasoning.

\paragraph{GenEval \cite{geneval}} is an object-focused framework designed to evaluate the compositional properties of text-to-image models. Unlike holistic metrics such as CLIPScore, GenEval leverages off-the-shelf object detection and segmentation models (e.g., Mask2Former) to verify whether generated images faithfully follow fine-grained instructions. It specifically decomposes prompts into six atomic tasks: single object presence, two-object co-occurrence, counting, color attribution, spatial positioning, and attribute binding. By providing binary correctness signals for each component, it offers an interpretable and granular assessment of a model's ability to handle complex semantic compositions.

\paragraph{GenEval++ \cite{Echo-4o}} serves as a more challenging extension of the original GenEval, introduced to address the issues of score saturation and evaluation inaccuracy in existing benchmarks. It significantly increases the complexity of instructions by coupling more objects with diverse attributes. Instead of relying solely on rule-based detectors, GenEval++ employs advanced vision-language models (e.g., GPT-4o) as evaluators to perform a checklist-based verification. A generation is only marked as successful if it satisfies a comprehensive set of criteria, including object count, color, relative position, and size, thereby providing a more rigorous measure of instruction-following fidelity.

\paragraph{ImagineBench \cite{Echo-4o}} focuses on the model's capability for surreal and imaginative generation, moving beyond the simple reproduction of real-world scenes. It requires models to synthesize "unknown" entities by augmenting common objects with fantastical elements while preserving their core identity features (e.g., "a square soccer ball"). The benchmark evaluates three key dimensions: fantasy fulfillment, identity preservation, and aesthetic quality. By emphasizing the tension between creative modification and identity consistency, ImagineBench provides a unique lens into the model's deep semantic understanding and creative synthesis potential.

\paragraph{PRISM-Bench \cite{PRISM-Bench}} (Precise and Robust Image Synthesis Measurement Benchmark) is a large-scale, multi-dimensional benchmark introduced alongside the FLUX-Reason-6M dataset. It comprises seven distinct evaluation tracks: Imagination, Entity, Text rendering, Style, Affection, Composition, and a particularly challenging Long Text track. The latter utilizes Generation Chain-of-Thought (GCoT) descriptions to test the model's ability to follow complex, multi-step reasoning instructions. PRISM-Bench leverages state-of-the-art VLMs for human-aligned assessment, ensuring a robust evaluation of both prompt-image alignment and visual aesthetics across a broad spectrum of generative tasks.

\paragraph{WiseBench \cite{niu2025wise}} (based on the WISE benchmark) is specifically designed to evaluate how well T2I models integrate and apply world knowledge during generation. It shifts the focus from shallow word-pixel mapping to deep semantic reasoning, challenging models with 1,000 meticulously crafted prompts across 25 subdomains, including cultural common sense, spatio-temporal reasoning, and natural sciences. The benchmark introduces the WiScore metric, which provides a weighted assessment of consistency, realism, and aesthetic quality. WiseBench is instrumental in revealing the "understanding-generation gap," where models might possess internal knowledge but struggle to manifest it accurately in synthesized pixels.

\subsection{Detail of Data Source and Pipeline} \label{sec:appendix_data_pipeline}

The construction of the CLVR training data follows a rigorous automated pipeline designed to synthesize high-quality reasoning trajectories. We utilize the \texttt{FLUX-Reason-6M} dataset as the primary source for initial prompts. To prevent data leakage and ensure evaluation integrity, we enforce strict isolation between training and probing data (\ref{sec:probe}), ensuring no overlap between the training trajectories and the evaluation benchmarks. Starting from approximately $10^5$ candidate prompts, our pipeline yields \textbf{20,861} high-quality trajectories---a retention rate of 20.9\%, after passing through a series of stringent filtering and quality control stages.

The data generation is orchestrated by a \textit{state-constrained agentic controller}, where \textbf{Gemini 2.5 Pro} \cite{Gemini25} serves as the VLM Controller for reasoning and state management, and Seedream 4 \cite{Seedream4} acts as the Diffusion Agent for image generation and refinement. The controller transitions through a predefined state machine: $\texttt{generate\_base\_image} \to \texttt{inspect} \to \texttt{edit/refine} \to \texttt{validate} \to \texttt{finalize}$. Each transition is governed by strict constraints, where the agent is limited to specific tools with validated input/output schemas. To ensure robustness, each state is assigned a fixed retry budget; trajectories that fail to converge within this budget or violate format constraints are immediately discarded.

To guarantee the logical and visual fidelity of the synthesized trajectories, we implement a multi-dimensional quality control mechanism. Central to this is a joint model consensus strategy, where both Gemini 2.5 Pro and Seed 1.8 \cite{seed2026seed18modelcardgeneralized} must concurrently validate the correctness of each CoT step and agree that the resulting CoT sequence yields a final image superior to a single-step baseline. Furthermore, we employ \textit{blind A/B testing} during the refinement phase to retain only the most optimal reasoning path. Any step exhibiting logical incoherence or quality degradation triggers a \texttt{FAIL(mid)} flag, leading to the immediate termination of the trajectory. The finalized data is exported in \texttt{ShareGPT} format, incorporating \texttt{rewrite} rules and \texttt{<IMG\_GEN\_n>} tokens to align multi-step reasoning with the target diffusion model's conditioning.

\begin{figure}[ht]
  \centering
  \includegraphics[width=\linewidth]{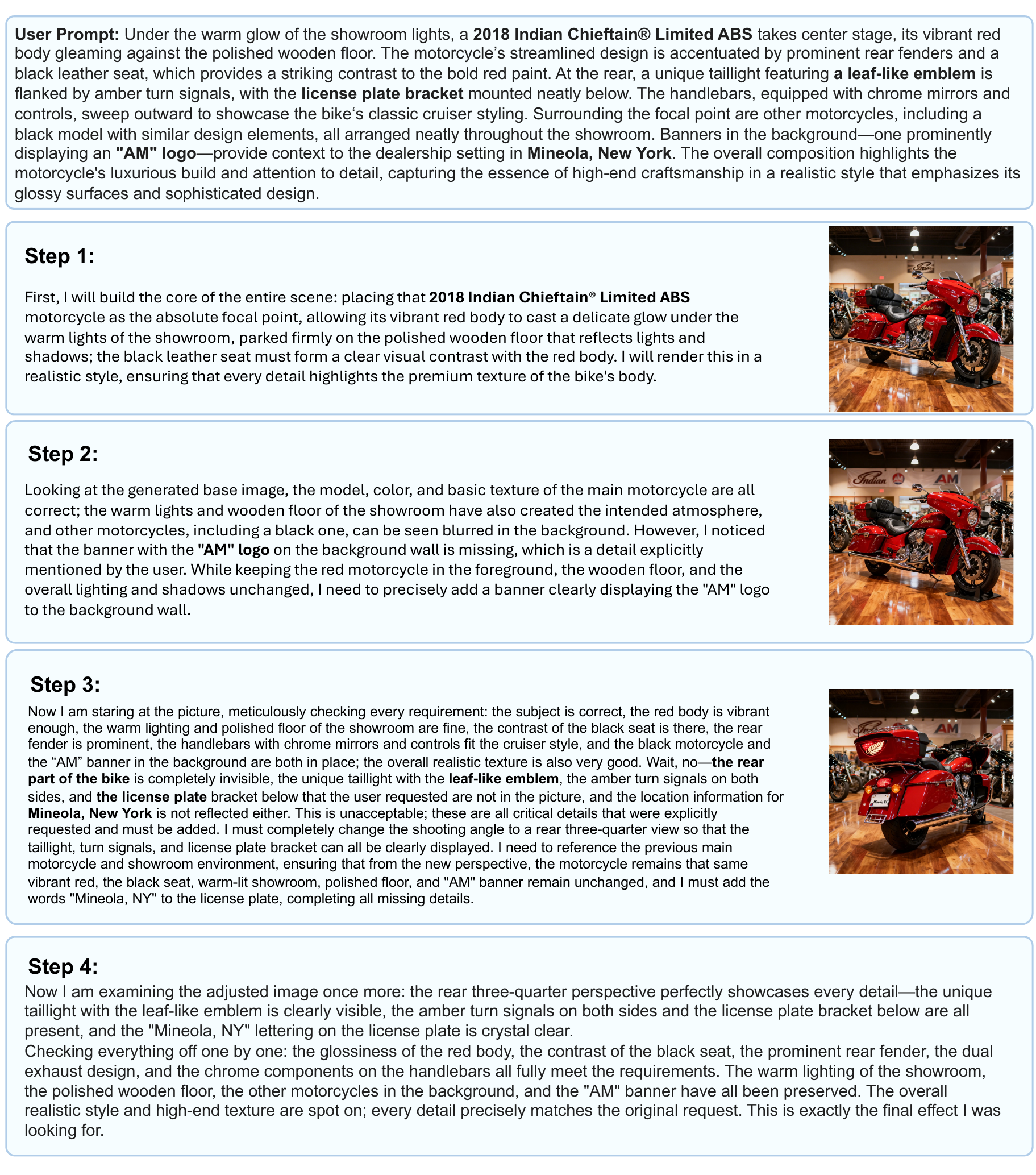}
  \caption{An example of training data generated by data synthesis pipeline. The system first generates a base image, then identifies a missing "AM" logo in the background and adds it. Finally, it recognizes that the rear details are missing and changes the perspective to a rear three-quarter view to include the taillight and license plate, demonstrating effective self-correction.}
  \label{fig:trace_showcase}
\end{figure}

Figure~\ref{fig:trace_showcase} provides a concrete visualization of a verified reasoning trajectory synthesized by our data engine. The example demonstrates how the controller iteratively refines the image based on visual feedback, transitioning through reasoning, editing, and verification states until the final objective is met.

\subsection{Additional Results on ImagineBench and GenEval++} \label{sec:appendix_additional_benchmarks}

In this section, we provide the detailed performance for GenEval++ and ImagineBench, as summarized in Table \ref{tab:exp-genevalpp-imagine-detailed}.

\subsection{Ablation Study on CLVR and PPRL Effectiveness}

Table~\ref{tab:ablation_RL} measures two effects: the benefit of CLVR itself without diffusion post-training, and the additional effect of diffusion post-training variants on top of CLVR.
The rows marked as "rewrite" use Qwen3-VL 8B only to rewrite the input prompt before a single generation pass, providing an open-loop semantic-enrichment baseline. The "VLM SFT only" CLVR setting means that the VLM controller is supervised fine-tuned for CLVR planning, while the diffusion model is not SFT-tuned or RL-aligned. Both simple RL and PPRL are applied after diffusion SFT. The simple RL baseline uses the same RL recipe as PPRL, but its reward model scores each rollout using only the reasoning text from the preceding round and, when images are present, only the most recent image. PPRL instead converts the interleaved multimodal history into proxy prompts, making the reward signal explicitly conditioned on the intended step-level visual goal.

First, CLVR provides clear gains on explicit compositional evaluation even without diffusion post-training. On GenEval, enabling CLVR with only VLM SFT improves the FLUX.2 4B Distill baseline from 0.81 to 0.86 overall, while open-loop prompt rewriting degrades it to 0.78. On WiseBench, however, most of the improvement over the distilled baseline comes from the VLM's knowledge injection through prompt rewriting (0.48 to 0.64), and the VLM-SFT-only CLVR setting further improves the score only modestly to 0.6577. This indicates that VLM SFT at this stage mainly teaches the controller to use CoT-style context and the closed-loop interface, but does not yet fully solve knowledge-intensive generation; stronger gains require aligning the diffusion model to follow these multi-step contextual instructions.

Second, diffusion post-training mainly helps when the prompt exceeds simple compositional matching. Diffusion SFT alone does not immediately improve WiseBench (0.65 to 0.62), because learning to consume long multimodal CoT contexts is substantially harder than fitting ordinary text-image pairs. We view this stage as an initialization step: it increases the probability that the diffusion model executes multi-step contextual instructions correctly, but is not sufficient to master them without RL optimization. Once followed by PPRL, WiseBench rises to 0.74, whereas the corresponding GenEval gain is smaller (0.86 to 0.87). This is expected because GenEval is both closer to saturation and built from relatively explicit prompts, while CLVR is designed for complex instructions that require multi-step decomposition, contextual following, and visual self-correction.

Third, the RL comparison isolates why the proxy reward is necessary. Simple RL is unstable after diffusion SFT: it drops on GenEval (0.85 to 0.84) and brings only a small WiseBench gain (0.62 to 0.64). In contrast, PPRL improves both benchmarks (0.87 on GenEval and 0.74 on WiseBench). These results support our diagnosis that directly rewarding long multimodal CoT rollouts is noisy: free-form CoT describes reasoning progress but does not always specify the visual objective in an explicit instruction format, making it difficult for the reward model to identify what should be evaluated. By converting each step into a proxy prompt, PPRL makes the reward target explicit and produces cleaner credit assignment.

\begin{table*}[t]
  \centering
  \small
  \caption{Ablation study of CLVR and PPRL on GenEval and WiseBench. "rewrite" denotes an open-loop prompt rewrite by Qwen3-VL 8B before single-pass generation. "VLM SFT only" means that only the VLM controller is supervised fine-tuned for CLVR planning, while the diffusion model is not SFT-tuned or RL-aligned. "diffusion SFT" denotes supervised fine-tuning of the diffusion model on CLVR trajectories. "simple RL" uses the same RL setup as PPRL but rewards rollouts from the previous-round reasoning text and the latest image only, while PPRL uses proxy prompts derived from the full multimodal trajectory to provide step-level reward targets. On GenEval, CLVR improves over both the distilled and rewrite baselines; on WiseBench, rewrite accounts for much of the initial knowledge-driven gain, while diffusion SFT+PPRL brings the largest additional improvement.}
  \label{tab:ablation_RL}
  
  \textbf{Panel A: ablation on GenEval.}
  \resizebox{\textwidth}{!}{
  \begin{tabular}{lccccccc}
  \toprule
  Setting & Single Object & Two Objects & Counting & Colors & Position & Color Attributes & Overall \\
  \midrule
  FLUX.2 4B Distill & 0.99 & 0.91 & 0.77 & 0.86 & 0.69 & 0.64 & 0.81 \\
  FLUX.2 4B Distill (rewrite) & 0.99 & 0.86 & 0.66 & 0.83 & 0.64 & 0.68 & 0.78 \\
  + CLVR (VLM SFT only) & 1.00 & 0.93 & 0.80 & 0.86 & 0.85 & 0.69 & 0.86 \\
  + CLVR + diffusion SFT & 1.00 & 0.95 & 0.81 & 0.86 & 0.81 & 0.67 & 0.85 \\
  + CLVR + diffusion SFT + simple RL & 1.00 & 0.93 & 0.79 & 0.89 & 0.78 & 0.65 & 0.84 \\
  + CLVR + diffusion SFT + PPRL & 0.99 & 0.92 & 0.85 & 0.89 & 0.85 & 0.71 & 0.87 \\
  \bottomrule
  \end{tabular}
  }
  
  \vspace{8pt}
  \textbf{Panel B: ablation on WiseBench.}
  \resizebox{\textwidth}{!}{
  \begin{tabular}{lccccccc}
  \toprule
  Setting & Cultural & Time & Space & Biology & Physics & Chemistry & Overall \\
  \midrule
  FLUX.2 4B Distill & 0.42 & 0.52 & 0.64 & 0.47 & 0.57 & 0.41 & 0.48 \\
  FLUX.2 4B Distill (rewrite) & 0.64 & 0.60 & 0.75 & 0.61 & 0.67 & 0.55 & 0.64 \\
  + CLVR (VLM SFT only) & 0.65 & 0.64 & 0.74 & 0.66 & 0.70 & 0.53 & 0.65 \\
  + CLVR + diffusion SFT & 0.62 & 0.62 & 0.73 & 0.60 & 0.62 & 0.54 & 0.62 \\
  + CLVR + diffusion SFT + simple RL & 0.62 & 0.61 & 0.74 & 0.65 & 0.68 & 0.54 & 0.64 \\
  + CLVR + diffusion SFT + PPRL & 0.73 & 0.72 & 0.83 & 0.73 & 0.80 & 0.64 & 0.74 \\
  \bottomrule
  \end{tabular}
  }
  \end{table*}

\subsection{Training Setup and Hyperparameters}
\label{sec:appendix_training_setup_en}

\paragraph{Supervised Fine-Tuning (SFT)}
For the visual-language model in our system, we perform full-parameter fine-tuning based on Qwen3-VL 8B. The training data consists of 20,861 processed trajectory samples, consistent with those described in Appendix \ref{sec:appendix_data_pipeline}. During training, we employ \texttt{bf16} mixed precision and a cosine learning rate scheduler. The primary hyperparameters include a learning rate of $1 \times 10^{-5}$, a warmup ratio of $0.1$, and a total of 3 training epochs. Regarding batch processing, the per-device batch size is set to 1 with gradient accumulation steps of 8.

In the supervised fine-tuning phase of the diffusion model, we utilize the same set of 20,861 trajectory metadata. Based on the FLUX.2 architecture, we perform full-parameter fine-tuning on the DiT (Diffusion Transformer) component. Key parameters are as follows: a learning rate of $2 \times 10^{-5}$, 3 training epochs, a per-device batch size of 1, and a training image resolution of $1024 \times 1024$. During training, a trajectory with $n$ steps is decomposed into $n$ individual training samples based on the number of images, where the training objective is the image at the end of each sample. The resulting checkpoint from this SFT phase serves as the foundation model for the subsequent reinforcement learning stage. 

\paragraph{Reinforcement Learning (RL) for Diffusion Model}
For the reinforcement learning phase, we utilize the DiffusionNFT algorithm for training, initialized with the weights saved during the SFT phase (SFT Warmup Checkpoint). This stage employs LoRA (Low-Rank Adaptation) for fine-tuning, with $\text{Rank} = 128$ and $\alpha = 256$. The training resolution is set to $512 \times 512$, and the number of sampling steps during rollout is 8 (Note: See Appendix \ref{sec:appendix_training_setup_en} for details on the evaluation metrics and unified protocols for the final model). Other critical training parameters include a learning rate of $1 \times 10^{-4}$, a KL penalty coefficient $\beta_{\text{KL}}$ of $1 \times 10^{-5}$, a per-device batch size of 1, a group size of 16, and a CFG (Classifier-Free Guidance) scale of 4.0. All training stages (SFT and RL) and the experiments reported in the main paper were run using 8 NVIDIA H20 GPUs.

Regarding the design of feedback signals, we configure a dual-path reward mechanism for T2I (Text-to-Image) and I2I (Image-to-Image) tasks:
\begin{itemize}
    \item The general reward for both T2I and I2I is provided by the \texttt{unifiedreward} model, which scores aesthetic quality and text-image alignment.
    \item The I2I editing reward is based on the \texttt{unifiedreward\_edit} model, providing precise feedback for multi-turn editing instructions (where the instruction key points to \texttt{proxy\_i2i\_prompt}).
\end{itemize}
To ensure a balanced training across multi-step chains of thought, the task mixing weight for T2I and I2I is set to $1:1$. Additionally, the sampling weights for different I2I step counts (steps 1, 2, 3, and $\ge 4$) are all set to equal proportions ($1:1:1:1$).

\subsection{Probe Study: Semantic Complexity Scaling}
\label{sec:probe_appendix}

This section details the Semantic Complexity Scaling Probe used to diagnose structural degradation in single-step T2I generation. This probe provides the empirical basis for introducing Closed-Loop Visual Reasoning (CLVR). All conclusions are specific to the complexity-stratified prompt set and evaluation protocols defined herein and should not be treated as universal theoretical guarantees.

\paragraph{Semantic Dependency Graph.}
For a prompt $P$, we map its requirements to a semantic dependency graph $G(P) = (\mathcal{V}, \mathcal{E}_{\mathrm{attr}}, \mathcal{E}_{\mathrm{rel}}, \mathcal{H})$. Here, $\mathcal{V}$ denotes entity instances, $\mathcal{E}_{\mathrm{attr}}$ represents attribute-to-entity bindings, $\mathcal{E}_{\mathrm{rel}}$ captures inter-entity relationships (spatial, action, or subordinate), and $\mathcal{H}$ contains hard constraints (e.g., style, text rendering) not captured by standard edges. For an entity group $j$ with $c_j$ instances and attributes $\mathcal{A}_j$, the total nodes $N$ and attribute edges $E_{\mathrm{attr}}$ are:
\begin{equation}
\label{eq:node_edge_counts}
N = \sum_{j=1}^{M} c_j, \quad E_{\mathrm{attr}} = \sum_{j=1}^{M} c_j \,\lvert \mathcal{A}_j \rvert.
\end{equation}
The total edge count is $E = E_{\mathrm{attr}} + \lvert \mathcal{E}_{\mathrm{rel}} \rvert$. Hard constraints $\mathcal{H}$ are handled separately via $R_{\mathrm{extra}}$.

\paragraph{Task Complexity Score.}
We define $C_{\mathrm{task}}(P)$ as a proxy for task difficulty, inspired by structural description length:
\begin{equation}
\label{eq:C_task_general}
C_{\mathrm{task}}(P) = \alpha N \log(1+N) + \beta E + \gamma_{w} \log(1+W) + R_{\mathrm{extra}},
\end{equation}
where $W$ is the word count. The term $N\log(1+N)$ accounts for the non-linear cost of organizing multiple entities, while $E$ represents linear constraint growth. Following our reproducible protocol, we set $\alpha = \beta = 1$. The term $R_{\mathrm{extra}}$ aggregates specific constraints:
\begin{equation}
\label{eq:R_extra}
R_{\mathrm{extra}} = \sum_{\substack{\mathrm{type}\,\in\,\mathcal{S}_{\mathrm{cnst}} \\ \mathcal{S}_{\mathrm{cnst}}=\{\mathrm{global},\,\mathrm{count},\, \\ \mathrm{text},\,\mathrm{neg}\}}} c_{\mathrm{type}}\, n_{\mathrm{type}},
\end{equation}
with empirical weights $c_{\mathrm{text}}=3.0$, $c_{\mathrm{count}}=2.0$, $c_{\mathrm{neg}}=1.5$, and $c_{\mathrm{global}}=0.5$, reflecting their relative difficulty for current diffusion models.

\paragraph{Stratification and Targeted Trimming.}
Prompts are partitioned into ten complexity tiers $\mathcal{T}_{01}$--$\mathcal{T}_{10}$ based on $C_{\mathrm{task}}$ quantiles. To decouple semantic complexity from text length, each tier is constrained to a specific word-count interval $[\ell_k, u_k]$. We employ Targeted Trimming (TRIM) to populate sparse tiers by iteratively removing secondary elements from high-complexity samples until they fall within the target $(C_{\mathrm{task}}, W)$ feasibility region, ensuring semantic naturalness.

\paragraph{Evaluation Protocol.}
We evaluate models (listed in Table~\ref{tab:probe_summary}) using $4$ images per prompt with fixed seeds $\mathcal{S} = \{42, 123, 456, 789\}$. A VLM-based judge evaluates fine-grained recall $r_{i,m}$ (fraction of satisfied constraints) and strict pass rate $p_{i,m}$ (all constraints satisfied). Results are aggregated across images before prompt-level statistics are computed.

\paragraph{AUC Metrics.}
To measure stability across the complexity axis, we calculate the Area Under the Curve (AUC) using trapezoidal integration:
\begin{equation}
\label{eq:auc_pass}
\mathrm{AUC}_{\mathrm{pass}}(m) = \sum_{k=1}^{9} \frac{x_{k+1}-x_k}{2} \left(y^{\mathrm{pass}}_{m,k}+y^{\mathrm{pass}}_{m,k+1}\right),
\end{equation}
where $x_k$ is the median $C_{\mathrm{task}}$ of tier $k$. AUC provides a more robust metric than single-threshold breakdown points by accounting for performance across the entire complexity spectrum.

\paragraph{Spectral Capacity Proxy.}
To characterize backbone complexity beyond raw parameter counts, we use the entropy effective rank $I_{\mathrm{eff}}$~\cite{I_eff}. This metric captures the usable expressive dimensions by analyzing the singular value distribution of weight matrices. For a matrix $\mathbf{W}$ with singular values $\{\sigma_i\}$, $I_{\mathrm{eff}}$ is derived from the entropy of the normalized energy distribution $p_i = \sigma_i^2 / \sum \sigma_j^2$. We report the median $r_{\mathrm{ent}}(\mathbf{W})$ across all core backbone layers.

\paragraph{Results.}
As shown in Table~\ref{tab:probe_summary}, single-step models follow a power-law relationship: $\mathrm{AUC}_{\mathrm{pass}} \propto I_{\mathrm{eff}}^{1.075}$ ($R^2=0.773, \rho=0.964$). CLVR (FLUX2) significantly deviates from this trend, achieving an $\mathrm{AUC}_{\mathrm{pass}}$ of $98.79$ compared to $73.89$ for its base model. This gain confirms that visual feedback and iterative refinement effectively bypass the capacity limitations of single-step generation.

\begin{table}[t]
\centering
\small
\caption{Semantic Complexity Scaling Probe results. $I_{\mathrm{eff}}$ serves as a spectral capacity proxy.}
\label{tab:probe_summary}
\begin{tabular}{lcrrrrr}
\hline
Model & DiT Params (B) & $I_{\mathrm{eff}}$ & $N$ & Pass & Median $C_{\mathrm{task}}$ & $\mathrm{AUC}_{\mathrm{pass}}$ \\
\hline
HunyuanDiT \cite{hunyuandit} & 1.5 & 514.48 & 992 & 0.129 & 11.68 & 17.67 \\
SD3.5-Med \cite{SD3.5} & 2.5 & 604.11 & 992 & 0.244 & 28.46 & 42.53 \\
AuraFlow \cite{auraflow} & 6.8 & 852.15 & 992 & 0.183 & 22.75 & 31.75  \\
SD3.5-Lrg \cite{SD3.5} & 8.0 & 977.86 & 992 & 0.288 & 34.65 & 53.79 \\
CogView3+ \cite{CogView} & 3.0 & 1063.40 & 992 & 0.292 & 39.56 & 55.71 \\
CogView4 \cite{CogView} & 6.0 & 1411.05 & 992 & 0.352 & 50.54 & 70.83  \\
FLUX.2 base \cite{flux-2} & 4.0 & 1586.36 & 992 & 0.372 & 50.32 & 73.89  \\
CLVR (Flux.2) & 4.0 & 1586.36 & 992 & 0.451 & 53.76 & 98.79 \\
\hline
\end{tabular}
\end{table}

\paragraph{Scope and Limitations.}
The probe's findings are specific to the chosen parser, $C_{\mathrm{task}}$ weighting, and VLM-judge protocol. It is designed for comparative analysis of relative trends rather than as a universal benchmark. While single-step models exhibit systematic degradation with increasing complexity, CLVR demonstrates that closed-loop reasoning can mitigate this collapse without scaling the underlying backbone's parameters.

\subsection{Inference Configuration, Sampling, and Trajectory Efficiency}
\label{sec:appendix_inference_protocol_en}

To ensure reproducibility, this section details the unified inference and sampling protocols employed across all benchmarks (GenEval, GenEval++, ImagineBench, PRISM, and WiseBench). We also analyze the generation efficiency and inference trajectory distribution of our method. A globally consistent configuration was maintained throughout all evaluations without per-benchmark tuning.

\paragraph{Maximum Closed-Loop Iterations}
During the Closed-Loop Visual Reasoning process, we strictly limit the maximum number of image generation cycles to 8 per task. Empirical observations across all benchmarks indicate that all tasks conclude within this 8-iteration limit. This upper bound effectively balances task success rates with computational overhead.

\paragraph{Diffusion Sampling Strategy}
We employ two standard decoding strategies tailored to different model variants:
\begin{itemize}
    \item \textbf{Base Decoding:} For the non-distilled base diffusion branch, we use 28 diffusion sampling steps with a Classifier-Free Guidance (CFG) scale of 4.0.
    \item \textbf{Distill Decoding:} For the accelerated distillation branch, we use 4-step rapid sampling without CFG (guidance scale set to 1.0). The deployed model, obtained via $\Delta$-Space Weight Merge (DSWM), uniformly follows this distillation configuration during inference.
\end{itemize}

\paragraph{Inference Acceleration and Trajectory Length Distribution}
\label{sec:appendix_inference_acceleration_en}
To quantify the acceleration provided by DSWM, we measure the end-to-end (E2E) latency on a server equipped with two NVIDIA H20 GPUs. The system is deployed using the vLLM framework, with the diffusion model and VLM controller allocated in a 1:1 ratio (one GPU each). Table~\ref{tab:geneval_step_timing_en} summarizes the E2E latency and trajectory distribution on the GenEval (553 test cases) and PRISM (700 test cases) benchmarks.

The results demonstrate that on GenEval, the majority of tasks (approximately 68\%) are successfully completed within 2 iterations, while 28\% require 3 iterations. Only a few challenging cases extend to 4--5 iterations, all of which remain well below the 8-iteration limit. Benefiting from the compatibility of DSWM with 4-step distillation, we significantly reduce latency while maintaining closed-loop visual feedback. For instance, in the most frequent 2-iteration trajectory, DSWM reduces the average E2E generation time from 287.0 seconds (Base) to 25.5 seconds, achieving approximately an 11$\times$ speedup.

Furthermore, comparing the sample distributions between the two benchmarks reveals that GenEval is relatively simple, requiring fewer reasoning steps. However, for its hard cases, CLVR effectively improves accuracy. In contrast, when evaluated on the more complex general benchmark (PRISM), the inference trajectory is noticeably longer. This demonstrates that our model possesses the capability to adaptively adjust its reasoning length based on the difficulty of the prompt.

\begin{table}[htbp]
\centering
\caption{Inference efficiency analysis. (a) Average end-to-end generation time (seconds) for different iteration counts. (b) Distribution of test cases across iteration counts on GenEval and PRISM benchmarks.}
\label{tab:geneval_step_timing_en}
\begin{minipage}[t]{0.42\textwidth}
\centering
\vspace{0pt}
\textbf{(a) Average E2E Generation Time (s)} \\
\vspace{4pt}
\begin{tabular}{ccc}
\toprule
\textbf{Iter.} & \textbf{\shortstack{Base\\(28 steps)}} & \textbf{\shortstack{DSWM\\(4 steps)}} \\
\midrule
1 & 192.4 & 12.6 \\
2 & 287.0 & 25.5 \\
3 & 471.7 & 38.2 \\
4 & 707.7 & 52.3 \\
5 & 950.2 & 68.7 \\
\bottomrule
\end{tabular}
\end{minipage}
\hfill
\begin{minipage}[t]{0.54\textwidth}
\centering
\vspace{0pt}
\textbf{(b) Sample Distribution} \\
\vspace{4pt}
\begin{tabular}{ccccc}
\toprule
\multirow{2}{*}{\textbf{Iter.}} & \multicolumn{2}{c}{\textbf{GenEval}} & \multicolumn{2}{c}{\textbf{PRISM}} \\
\cmidrule(lr){2-3}\cmidrule(lr){4-5}
& \textbf{Base} & \textbf{DSWM} & \textbf{Base} & \textbf{DSWM} \\
\midrule
1 & 2 & 2 & 4 & 3 \\
2 & 368 & 376 & 222 & 282 \\
3 & 158 & 154 & 363 & 351 \\
4 & 24 & 17 & 78 & 52 \\
5 & 1 & 4 & 25 & 10 \\
6 & 0 & 0 & 8 & 2 \\
\bottomrule
\end{tabular}
\end{minipage}
\end{table}

\subsection{Statistical Uncertainty for Primary Benchmarks}
\label{sec:checklist_statistics}

\noindent
This section documents \emph{standard errors} (SE) and \emph{nominal $95\%$ confidence intervals} for the Overall metrics on GenEval~\cite{geneval} and WiseBench~\cite{niu2025wise} for our method, under the same evaluation protocol as the main tables (Appendix~\ref{sec:appendix_inference_protocol_en}). It is provided for reproducibility checklist reporting.

\paragraph{GenEval.}
The Overall score is the fraction of prompts that pass automated verification on the official GenEval split ($N{=}553$ prompts). We report the Wald standard error of a binomial proportion,
\begin{equation}
  \mathrm{SE}_{\mathrm{bin}} = \sqrt{\frac{\hat{p}(1-\hat{p})}{N}},
\end{equation}
and $95\%$ \textbf{Wilson score} confidence intervals for $\hat{p}$ (recommended for moderate $N$ and $\hat{p}$ near $0$ or $1$).

\paragraph{WiseBench.}
The Overall WiScore is the mean score over the official WiseBench test set ($N{=}1000$ prompts; scores per prompt are in $\{0,1,2\}$ as defined by the benchmark). The standard error is the standard error of the mean, $\mathrm{SE}=s/\sqrt{N}$, where $s$ is the sample standard deviation of per-prompt scores from the same evaluation run. Nominal $95\%$ intervals use the normal approximation $\bar{x}\pm 1.96\,\mathrm{SE}$ (adequate at $N{=}1000$).

\begin{table}[t]
  \centering
  \small
  \caption{CLVR Overall scores with standard errors and $95\%$ confidence intervals. GenEval: Wilson intervals on $N{=}553$ prompts. WiseBench: normal-approximation intervals from $\mathrm{SE}=s/\sqrt{N}$ on $N{=}1000$ prompts.}
  \label{tab:checklist_se_ci}
  \begin{tabular}{llccccc}
    \toprule
    \textbf{Model} & \textbf{Benchmark} & \textbf{Metric} & \textbf{Point est.} & \textbf{$N$} & \textbf{SE} & \textbf{$95\%$ CI} \\
    \midrule
    CLVR (4B) & GenEval & Overall pass & $0.87$ & $553$ & $0.0146$ & $[0.8333,\,0.8904]$ \\
    CLVR (4B) & WiseBench & Overall WiScore & $0.7405$ & $1000$ & $0.0093$ & $[0.7223, \, 0.7588]$ \\
    \midrule
    CLVR (9B) & GenEval & Overall pass & $0.88$ & $553$ & $0.0143$ & $[0.8403,\,0.8964]$ \\
    CLVR (9B) & WiseBench & Overall WiScore & $0.7584$ & $1000$ & $0.0091$ & $[0.7405, \, 0.7763]$ \\
    \bottomrule
  \end{tabular}
\end{table}

\subsection{Qualitative Study: More Case Showcases}
\label{sec:appendix_quality_study}

In this section, we provide additional qualitative results to further demonstrate the effectiveness of our Closed-Loop Visual Reasoning (CLVR) framework across various complex scenarios. Figure~\ref{fig:gallery_combined} showcases the model's ability to handle intricate attribute bindings, spatial relationships, and iterative refinements.

\subsection{Limitations and Future Work}
\label{sec:limitations}

While the proposed Closed-Loop Visual Reasoning (CLVR) framework achieves a significant system-level breakthrough in complex text-to-image generation, we acknowledge several limitations that provide fertile ground for future investigation.

\paragraph{User-Controllable Inference Budget.}
CLVR significantly enhances generation quality for complex semantics through multi-turn visual feedback. However, this adaptive iterative reasoning inevitably introduces additional computational and temporal overhead. In the current system design, the number of feedback loops and the termination criteria are autonomously determined by the model based on the state of the generated image. Consequently, it is difficult for users to explicitly intervene in the trade-off between quality and cost. Future work could explore the implementation of a "thinking budget" control interface. This would allow users to manually specify the maximum number of reasoning steps, the frequency of visual feedback, or the upper bound of computational resources based on task difficulty, latency constraints, or specific quality requirements, thereby returning control over the inference budget to the user.

\paragraph{Expansion to Diverse Modalities and Scenarios.}
The focus of this study is primarily on closed-loop visual reasoning within the domain of static image generation. Nevertheless, the closed-loop visual reasoning paradigm is inherently generalizable and can be naturally extended to other modalities and scenarios. Potential applications include consistent multi-image generation, long-form video synthesis, 3D asset creation, and interactive design workflows. In these contexts, closed-loop reasoning will encounter novel challenges such as temporal consistency, cross-view geometric constraints, and dynamic shifts in user preferences. These areas represent promising directions for subsequent research.

\subsection{Broader Impacts}
\label{sec:broader_impacts}

\noindent
Text-to-image systems that better follow complex instructions can support creative production, prototyping, and inclusive communication. They also raise familiar risks: synthetic imagery may be misused for deception, impersonation, or harassment, and stronger semantic control could amplify such misuse if deployed without safeguards. CLVR contributes a closed-loop reasoning methodology rather than a consumer-facing product; we do not study moderation or release policy here. Future integrations could align iterative verification with organizational content policies and accountability tooling. We encourage layered mitigations, including disclosure of synthetic content, abuse monitoring, safety classifiers, and organizational review for high-stakes deployments, alongside continued research on robust media provenance.


\end{document}